%% file: SkipNode.tex
\pdfoutput=1%For ARXIV
\documentclass[10pt,journal,compsoc]{IEEEtran}

\usepackage[nocompress]{cite}
\usepackage{graphicx}
\usepackage{url}
\usepackage{threeparttable}
\usepackage[ruled]{algorithm2e}
\usepackage{algorithmic}
\usepackage{diagbox}
\usepackage{booktabs}
\usepackage{multirow}
\usepackage{multicol}
\usepackage{amsmath}
\usepackage{amssymb}
\usepackage{amsthm}
\usepackage{amsfonts}       % blackboard math symbols
\usepackage{nicefrac}       % compact symbols for 1/2, etc.

\usepackage{array}
\newtheorem{theorem}{Theorem}
\newtheorem{remark}{Remark}
\newtheorem{lemma}{Lemma}
\newtheorem{definition}{Definition}
\newtheorem{assumption}{Assumption}
\allowdisplaybreaks
\usepackage{subfigure}

\usepackage{xcolor}
\newcommand{\DMone}[1]{d_{\mathcal{M}}(#1)}
\newcommand{\DMtwo}[1]{d^{2}_{\mathcal{M}}(#1)}

\newcommand{\skipnode}{\texttt{SkipNode}\@\xspace}
\newcommand{\skipnodeu}{\texttt{SkipNode-U}\@\xspace}
\newcommand{\skipnodeb}{\texttt{SkipNode-B}\@\xspace}
\newcommand{\dropedge}{\texttt{DropEdge}\@\xspace}
\newcommand{\dropnode}{\texttt{DropNode}\@\xspace}
\newcommand{\dropnodes}{\texttt{DropNode-S}\@\xspace}
\newcommand{\dropnodef}{\texttt{DropNode-F}\@\xspace}
\newcommand{\dropmessage}{\texttt{DropMessage}\@\xspace}
\newcommand{\skipconnection}{\texttt{SkipConnection}\@\xspace}
\newcommand{\pairnorm}{\texttt{PairNorm}\@\xspace}

\newcommand{\para}[1]{\vspace{2mm} \noindent \textbf{#1}}

\newcommand{\lwgrevise}[1]{\textcolor{black}{#1}}

\begin{document}

\title{SkipNode: On Alleviating Performance Degradation for Deep Graph Convolutional Networks}

\author{Weigang~Lu,
        Yibing~Zhan,
        Binbin~Lin$^{*}$,
        Ziyu~Guan$^{*}$\thanks{* Corresponding author},
        Liu~Liu,
        Baosheng~Yu,
        Wei~Zhao,
        Yaming~Yang,
        Dacheng~Tao,~\IEEEmembership{Fellow,~IEEE}
        
\IEEEcompsocitemizethanks{
\IEEEcompsocthanksitem W. Lu, Z. Guan, W. Zhao, and Y. Yang are with the School of Computer Science and Technology, Xidian University, Xi'an, China 710071.
E-mail: \{wglu@stu., zyguan@, ywzhao@mail., yym@\}xidian.edu.cn

\IEEEcompsocthanksitem B. Lin is with the School of Software Technology, Zhejiang University, China.
E-mail: binbinlin@zju.edu.cn

\IEEEcompsocthanksitem L. Liu and B. Yu are with the School of Computer Science, the University of Sydney, Australia.
E-mail: \{liu.liu1, baosheng.yu\}@sydney.edu.au

\IEEEcompsocthanksitem Y. Zhan and D. Tao are with JD Explore Academy, China.
E-mail: zhanyibing@jd.com, dacheng.tao@gmail.com 
}%
\thanks{Manuscript received Oct.~2022.}
}

% \markboth{Journal of \LaTeX\ Class Files, ~Vol.~XX, No.~XX, Dec.~2020}%
\markboth{IEEE Transactions on Knowledge and Data Engineering,~Vol.~xx, No.~xx, Oct.~2022}%
{Lu \MakeLowercase{\textit{et al.}}: SkipNode: On Alleviating Performance Degradation for Deep Graph Convolutional Networks}

\IEEEtitleabstractindextext{%

\begin{abstract}
\lwgrevise{Graph Convolutional Networks (GCNs) suffer from performance degradation when models go deeper. However, earlier works only attributed the performance degeneration to over-smoothing. In this paper, we conduct theoretical and experimental analysis to explore the fundamental causes of performance degradation in deep GCNs: over-smoothing and gradient vanishing have a mutually reinforcing effect that causes the performance to deteriorate more quickly in deep GCNs. On the other hand, existing anti-over-smoothing methods all perform full convolutions up to the model depth. They could not well resist the exponential convergence of over-smoothing due to model depth increasing. In this work, we propose a simple yet effective plug-and-play module, \skipnode, to overcome the performance degradation of deep GCNs. It samples graph nodes in each convolutional layer to skip the convolution operation. In this way, both over-smoothing and gradient vanishing can be effectively suppressed since (1) not all nodes'features propagate through full layers and, (2) the gradient can be directly passed back through ``skipped'' nodes. We provide both theoretical analysis and empirical evaluation to demonstrate the efficacy of \skipnode and its superiority over SOTA baselines.}

\end{abstract}

}

\maketitle

\IEEEdisplaynontitleabstractindextext

\IEEEpeerreviewmaketitle

\IEEEraisesectionheading{\section{Introduction}\label{sec:intro}}
\input{./chapters/introduction.tex}

\section{Related Work}
\label{sec:related-work}
\input{./chapters/related_works.tex}

\section{Preliminaries}
\label{sec:preliminaries}
\input{./chapters/preliminaries.tex}

\section{Why Deep GCNs Fails?}
\label{sec:over-smoothing}
\input{./chapters/motivation.tex}

\section{Method}
\label{sec:skipnode}
\input{./chapters/skipnode.tex}

\section{Experiments}
\label{sec:exp}
\input{./chapters/experiment.tex}

\section{Conclusion}
In this study, we explore the fundamental causes of performance degradation in deep GCNs from a novel perspective: over-smoothing and gradient vanishing intensify each other, further degrading the performance. Based on theoretical and empirical analysis, we propose an effective and general framework to address this problem by reducing the convergence speed of over-smoothing and gradient vanishing. Besides, extensive experiments demonstrate that it is beneficial to apply \skipnode to deep GCNs.

\input{chapters/appendix}

\section*{Acknowledgments}
This research was supported by the National Natural Science Foundation of China (Grant Nos. 62133012, 61936006, 62303366, 62103314, 62073255), the Key Research and Development Program of Shaanxi (Program No. 2020ZDLGY04-07), the Innovation Capability Support Program of Shaanxi (Grant No. 2021TD-05), and the High-Performance Computing Platform of Xidian University.

\ifCLASSOPTIONcaptionsoff
  \newpage
\fi

\bibliographystyle{IEEEtran}
\bibliography{SkipNode}
\input{chapters/biography.tex}

\end{document}

%% file: chapters/introduction.tex
\IEEEPARstart{G}{raph} Convolutional Networks (GCNs) can capture the dependence in graphs through message passing between nodes. Due to the excellent ability for graph representation learning, GCNs have shown remarkable breakthroughs in numerous applications~\cite{zhou2020graph,soft}, including node classification~\cite{gcn}, object detection~\cite{Yao_2018_ECCV}, and visual question answering~\cite{Li_2019_ICCV}. However, despite the widespread success, current GCNs still suffer from a severe problem, i.e., the performance seriously degenerates with the increasing number of layers. 

The performance degradation problem is widely believed to be caused by the over-smoothing issue~\cite{li2018deeper,mad,yan2021two, wang2019improving, sgc, yang2020revisiting, pairnorm, oono2019graph, cai2020note, dropedge, appnp}. Specifically, each graph convolutional operation tends to mix the features of connected nodes through message propagation. The features between connected nodes quickly become indistinguishable when stacking too many graph convolutional operations. This phenomenon is referred to as the \textit{over-smoothing} issue. \cite{oono2019graph} also investigates the asymptotic behaviors of GCNs and finds that the node representations will approach an invariant space when the number of layers goes to infinity. \lwgrevise{Considering that many real-world graphs require modeling of long-range dependencies~\cite{nodemixup} and deeper neural networks usually show better expressive and reasoning ability~\cite{raghu2017expressive}, how to devise deeper GCNs by addressing the over-smoothing issue has received increasing attention from the community~\cite{li2019deepgcns,dropedge}. However, over-smoothing is not the sole reason for the performance degradation when deepening GCNs. Another important issue is the well-known gradient vanishing for deep models. In this study, we demonstrate that the story for GCNs is a bit different: over-smoothing and gradient vanishing mutually reinforce each other, making the performance degrade more rapidly compared to ordinary deep models. Specifically, we show that over-smoothing pushes the gradient of the classification loss w.r.t the output layer towards 0 at the very beginning of the training process (after the first forward propagation). This not only aggravates the traditional back-propagation-induced gradient vanishing issue, but also triggers the \textit{weight over-decaying} issue due to the dominating effect of regularization. The gradient vanishing and weight over-decaying issues in turn encourage the output features to further approach 0 (weight matrices with small values are multiplied repeatedly), thus worsening the over-smoothing issue and making the output trivial.}

To cope with the performance degradation problem, several GCN variants are proposed recently for alleviating the over-smoothing issue~\cite{jknet, inceptgcn, appnp, gprgnn, gcnii, dropnode_f,zhang2021evaluating,liu2020towards,wang2021dissecting,cong2021provable}. However, these methods still suffer from the following drawbacks: (1) most of these models try to alleviate over-smoothing by sacrificing the expressiveness merit of deep learning, e.g., emphasizing low-level features when over-smoothing happens~\cite{jknet, inceptgcn, appnp}, removing nonlinearity in middle layers~\cite{appnp}, and aggregating neighborhoods of different orders in one layer of computation~\cite{dropnode_f}; (2) some methods only deal with the gradient vanishing issue~\cite{he2016deep}, which has been proved in~\cite{oono2019graph} to be ineffective for the over-smoothing issue; and (3) all of these methods propose ad-hoc GCN models for addressing the over-smoothing issue. These existing techniques are not general enough to be applied to various GCN models. Considering the rapid development of the graph neural network field, a general technique that can be applied to different GCN models would be more beneficial to the research community. \dropedge~\cite{dropedge}, \dropnode~\cite{dropnode_s,dropnode_s}, and \pairnorm~\cite{pairnorm} are three such general modules against the over-smoothing issue. However, they mainly focus on improving feature diversity among different nodes by modifying the graph topology or adjusting the intermediate features, leaving other issues such as gradient vanishing poorly investigated.

More importantly, all the previous anti-over-smoothing methods still require performing $L$ full convolutional operations for each node, where $L$ is the model depth of a GCN-based model. Each convolutional operation multiplies the current node feature matrix by the normalized adjacency matrix and a weight matrix. \lwgrevise{Let $\lambda$ be the second largest magnitude of the eigenvalues of the normalized adjacency matrix and $s$ be the maximum singular value of weight matrices.} Theorem 2 from~\cite{oono2019graph} shows that the node feature matrix will \textit{exponentially} become over-smoothing as $L$ increases with convergence speed as $O((s\lambda)^{L})$, since usually $s\lambda < 1$. Although the previous methods try to alleviate over-smoothing with various ideas, none of them get out of the curse of $(s\lambda)^L$. As $L$ increases, the feature matrix still tends to degenerate due to the exponential effect of $(s\lambda)^L$.

\lwgrevise{A straightforward idea for simultaneously alleviating over-smoothing and gradient vanishing is to combine an existing anti-over-smoothing method with \skipconnection~\cite{residual}, a well-known strategy for the gradient vanishing issue. Nevertheless, although this simple combination is effective for gradient vanishing, it could not well cope with the over-smoothing issue, since $L$ full convolutions are still required. Previous work~\cite{oono2019graph} has also proved that \skipconnection contributes little to addressing over-smoothing.}

From the above analysis, we can see the key underlying issue of over-smoothing is the exponential effect of $(s\lambda)^{L}$ which is hard to break. In this paper, we take a different route to cope with the over-smoothing issue: rather than trying to increase $s$ or $\lambda$, we propose to decrease the exponent part, i.e., $L$. In particular, we devise a new plug-and-play module, dubbed \skipnode, to better train deep GCN-based models. \skipnode samples nodes in each convolutional layer to completely skip this layer's calculation. In other words, the features of those selected nodes are kept unchanged during the convolutional operation. \skipnode is analogous to Dropout~\cite{dropout}, where the difference is that we randomly drop layers, rather than features. Intuitively, this training strategy could effectively reduce $L$ for those selected nodes, and consequently would significantly alleviate the exponential effect of $(s\lambda)^{L}$; the skip operation naturally facilitates gradient back-passing, which means it could also well cope with gradient vanishing. We further show that even for one layer's calculation, \skipnode can generate better anti-over-smoothing results in expectation than vanilla GCN. The main contributions of this work are specified as follows:

%from the theoretical perspective, we can also in expectation increase the base part (i.e., $s\lambda$) to suppress the exponential effect of $(s\lambda)^{L}$.

\textbullet\
\textbf{We provide a more complete view of the causes of performance degradation in deep GCNs.} We analyze in detail two additional issues, gradient vanishing and model weights over-decaying, and their relationships with over-smoothing. The interplay between them is responsible for the terrible performance when deepening GCNs.

\textbullet\
\textbf{We devise a novel plug-and-play module, dubbed \skipnode, for addressing performance degradation in deep GCNs.} In each middle layer of a GCN model, \skipnode samples nodes either uniformly or based on node degrees to skip the convolutional operation completely. For unselected nodes, they can still receive information from these selected (skipped) ones. We analyze the effectiveness of \skipnode towards alleviating the above three issues and prove that it can better resist over-smoothing even considering one layer's calculation (i.e., it can improve both exponent part and base part of $(s\lambda)^{L}$).

\textbullet\
\textbf{Extensive experiments are conducted to verify the generalizability and effectiveness of \skipnode.} 
The experimental results on various datasets demonstrate the good capability of our \skipnode to alleviate performance degradation in deep GCNs. Specifically, (1) \skipnode has strong generalizability to improve GCN-based methods (including those designed for the over-smoothing issue) on various graphs and tasks, and (2) \skipnode consistently outperforms the commonly-used plug-and-play strategies in different settings.

%% file: chapters/related_works.tex
\begin{table*}[!th]
    \centering
    \caption{\lwgrevise{Comparison of Existing Plug-and-play Strategies.}}

\begin{tabular}{c | c c c c c c c}
\toprule
Issues & \pairnorm & \dropedge & \dropnodef & \dropnodes & \dropmessage & \skipconnection & \textbf{\skipnode} \\
\midrule

over-smoothing & \checkmark & \checkmark & \checkmark & \checkmark & \checkmark & $\times$ & \checkmark \\
gradient vanishing & $\times$ & $\times$ & $\times$ & $\times$ & $\times$ & \checkmark & \checkmark \\
weight over-decaying & $\times$ & $\times$ & $\times$ & $\times$ & $\times$ &\checkmark & \checkmark \\

\bottomrule
\end{tabular}

\label{tab:comparison}
\end{table*}

In this section, we first briefly review the development of graph convolution networks, and then introduce existing works targeting the over-smoothing issue. Finally, we discuss the differences between our \skipnode and several closely related strategies, i.e., \pairnorm\cite{pairnorm}, \dropedge\cite{dropedge}, \dropnode\cite{dropnode_s, dropnode_f}, \dropmessage\cite{dropmessage}, and \skipconnection\cite{residual}. 

\subsection{Graph Convolutional Networks} \cite{bruna2013spectral} first devises the graph convolution operation based on the graph spectral theory. However, considering the high computational cost, it is very difficult to apply spectral-based models~\cite{bruna2013spectral,henaff2015deep, defferrard2016convolutional,gcn} to very large graphs. Spatial-based models~\cite{sage,monti2017geometric,gat,niepert2016learning} exhibit great scalability since their convolutional operations can be directly applied to spatial neighborhoods of unseen graphs. \cite{gcn} simplifies previous works and proposes the current commonly-used graph convolutional network, i.e., GCN. Based on~\cite{gcn}, many works~\cite{gin,you2019position,9336270,sgc,abu2020n} toward improving graph learning have been proposed. \cite{gin} investigates the expressive power of different aggregations and proposes a simple model, GIN. P-GNN~\cite{you2019position} incorporates positional information into GCN. SGC~\cite{sgc} reduces the complexity of GCN by discarding the nonlinear function and weight matrices. Despite performance improvement, deepening GCNs remains a challenging problem. \lwgrevise{While our \skipnode shares the similar concept of sampling node used in some inductive GCNs, such as GraphSAGE~\cite{sage} and TGAT~\cite{tgat}, these GCNs focus on neighborhood sampling for generating inductive node embeddings instead of alleviating performance degeneration in deep GCNs.}

\subsection{Over-smoothing}
Over-smoothing is a basic issue of deep GCNs~\cite{li2018deeper,survey_oversmoothing}. \cite{yang2020revisiting, appnp, sgc} find that the node representation becomes linearly dependent after using an infinite number of layers. \cite{oono2019graph,cai2020note} further proves that the output of deep GCNs with a ReLU activation function converges to zero. \cite{yan2021two} investigates the over-smoothing issue in heterophilic and homophilic graphs. \lwgrevise{To quantify over-smoothing, the Mean Average Distance (MAD) method proposed by Chen \textit{et al.}~\cite{mad} calculates the mean of the average cosine distance between nodes and their connected nodes.} In general, current solutions to the over-smoothing issue fall into two categories:

\para{Specific GCNs-based models for over-smoothing}:
\cite{gcn} attempts to construct a multi-layer GCN by adding residual connections \cite{he2016deep}. JKNet~\cite{jknet} combines the outputs of different layers for the final representation. InceptGCN~\cite{inceptgcn} integrates different outputs from multiple GCNs with various layers. GCNII~\cite{gcnii} obtains remarkable performance over many benchmarks by utilizing initial residual connection and identity mapping. APPNP~\cite{appnp} incorporates the PageRank algorithm into GCN to derive a novel propagation schema instead of the Laplacian smoothing. \cite{appnp} follows PageRank-style propagation and adds adaptive weights to escape from over-smoothing guided by label information.~\cite{zhang2021evaluating,liu2020towards,wang2021dissecting,cong2021provable} decouple feature transformation from propagation to avoid over-smoothing (and model degeneration). \cite{zhou2021dirichlet} proposes a novel framework, consisting of orthogonal weight controlling, lower-bounded residual connection, and shifted ReLU activation, to train deep GCNs by regularizing Dirichlet energy. However, the extra efforts on weight regularization increase the complexity of optimization. \lwgrevise{Furthermore, Miao \textit{et al.} develops Lasagne~\cite{lasagne}, which incorporates node-aware layer aggregators and factorization-based layer interactions as a solution to address over-smoothing. Rusch \textit{et al.}~\cite{graphcon} proposes the Graph-Coupled Oscillator Network (GraphCON) as a solution to over-smoothing. By utilizing non-linear oscillators coupled through the graph structure, GraphCON avoids the exponential vanishing of Dirichlet energy. Another approach, Allen-Cahn Message Passing (ACMP)~\cite{acmp}, as introduced by Wang \textit{et al.}, is based on the Allen-Cahn equation. ACMP models interacting particle systems with attractive and repulsive forces, thereby offering a distinct mechanism to mitigate over-smoothing.} However, the above models are precisely designed and the proposed techniques might struggle to be used on various GCN models.

\para{Plug-and-play techniques for over-smoothing}: \lwgrevise{Several normalization-based methods have been proposed to address the issue of over-smoothing in GNNs. Yang \textit{et al.}~\cite{yang2020revisiting}, \textit{Oono et al.}~\cite{oono2019graph}, and Zhao \textit{et al.}~\cite{pairnorm} have introduced various techniques such as mean-subtraction, node representations/weight scaling, and center-and-scale normalization. These methods aim to prevent node features from becoming indistinguishable during the propagation process. Among these techniques, \pairnorm~\cite{pairnorm} has gained significant attention for re-normalizing intermediate node representations and mitigating over-smoothing. Another rapidly emerging strategy to mitigate over-smoothing involves integrating augmented data into models through the random "dropping" of certain input components. Rong \textit{et al.} proposes \dropedge~\cite{dropedge} to alleviate both over-smoothing and over-fitting issues by randomly deleting a part of edges. In addition to edge-based methods, \dropnode approaches focus on nodes by either masking their features~\cite{dropnode_f} or completely removing them from the graph~\cite{dropnode_f}. To differentiate between these \dropnode methods, we refer to~\cite{dropnode_f} as \dropnodef and~\cite{dropnode_s} as \dropnodes. What's more, \dropmessage~\cite{dropmessage} directly drops the propagated messages during the message-passing process.} Unfortunately, these works excessively emphasize maintaining feature diversity but ignore other issues beyond over-smoothing. In summary, none of the previous methods dealing with deep GCNs' performance degradation tried to alleviate both over-smoothing and gradient vanishing \lwgrevise{(cf. Figure~\ref{fig:negative_effects_on_cora})}. Moreover, they still perform graph convolutions up to the model depth, which could not well resist the exponential smoothing effect of model depth. 

\subsection{\lwgrevise{Detailed Comparison to Plug-and-play Strategies}}

\begin{figure*}[!ht]
    \centering
    \includegraphics[width=1\textwidth]{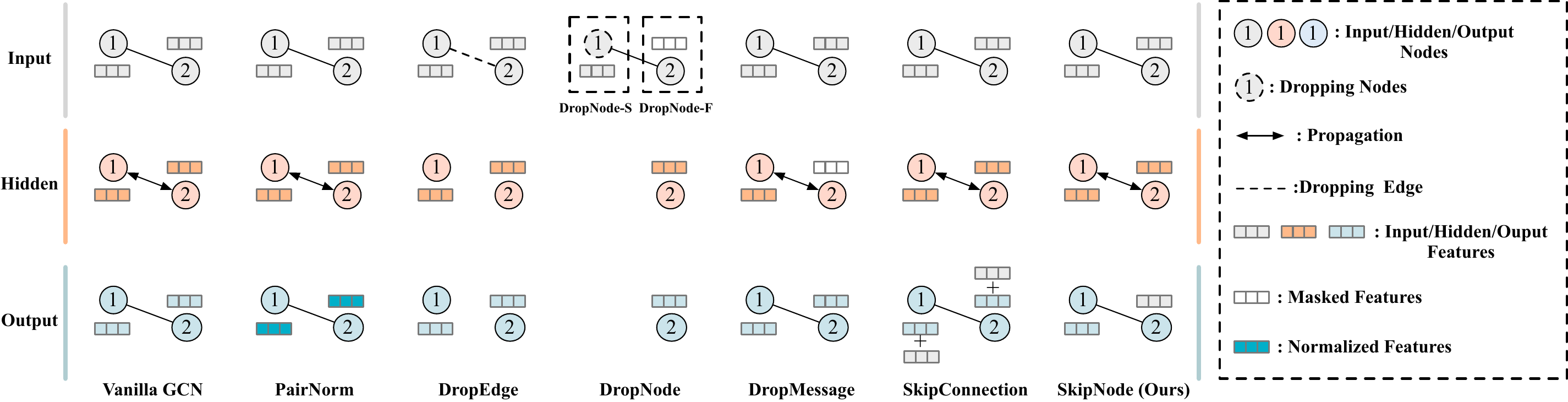}
    \caption{\lwgrevise{The sketch of how vanilla GCN and GCN with \pairnorm/\dropedge/\dropnode/\dropmessage/\skipconnection/\skipnode perform on a graph. Vanilla GCN enables each connected node to exchange information. \pairnorm renormalizes the output features. \dropedge deletes the edge between nodes 1 and 2 to make the graph sparser. \dropnodes removes node 1, while \dropnodef masks the entire input feature of node 2. \skipconnection directly adds the input on the output. \skipnode replaces the output feature of node 2 with its input.}}
    \label{fig:difference}
\end{figure*}

\lwgrevise{Figure~\ref{fig:difference} illustrates the key differences of the several plug-and-play strategies when applied to one layer of GCNs. It depicts how they manipulate the input node features and the graph structure (the first row of Figure~\ref{fig:difference}), aggregate features from neighbors (the second row), and generate the output features (the third row).} The last column shows that \skipnode replaces the sampled node 2's output by the respective input node feature. As a result, this ``skip'' operation not only well resists the exponential smoothing effect of model depth by skipping some convolutional operations for the sampled nodes, but also enables the gradient to be transferred via the sampled nodes, hence alleviating gradient vanishing. Finally, because the above two issues have been well alleviated, the weight norm will not degrade excessively. \lwgrevise{In the following, we discuss in detail the intrinsic differences between \skipnode and \pairnorm, \dropedge, \dropnode, \dropmessage, and \skipconnection. In Table~\ref{tab:comparison}, we also summarize their ability to deal with the three issues, namely, over-smoothing, gradient vanishing, and weight over-decaying.}

\para{\skipnode \textit{vs.} \pairnorm.}
\lwgrevise{In the second column of Figure~\ref{fig:difference}, \pairnorm renormalizes the output features of both nodes 1 and 2 to prevent them from becoming too similar. However, \pairnorm overlooks the issues of gradient vanishing and weight over-decaying. As a result, it may encounter difficulties when applied to deeper GCNs (smaller than 30 layers) with empirical evidence shown in~\cite{pairnorm}.}

\para{\skipnode \textit{vs.} \dropedge.}
As shown in the third column of Figure~\ref{fig:difference}, \dropedge deletes the edge between nodes 1 and 2 during training to keep them from being too similar. The key idea of \dropedge is that more sparsity of the graph results in less over-smoothing. However, randomly deleting too many edges may compromise critical structural information, while removing too few edges cannot effectively relieve over-smoothing. As a result, determining the dropping ratio when \dropedge is applied to various graphs is not a simple task. \dropedge primarily concentrates on over-smoothing, ignoring gradient vanishing and weight over-decaying issues since the gradient still needs to go through the model layer by layer and inevitably end up vanishing. Moreover, \dropedge re-normalizes the adjacency matrix in each epoch, which is time-consuming when dealing with large graphs.

\para{\skipnode \textit{vs.} \dropnode.}
\lwgrevise{Two works~\cite{dropnode_f,dropnode_s} have introduced the \dropnode technique, as shown in the fourth column of Figure~\ref{fig:difference}. In~\cite{dropnode_f}, the ``dropping'' strategy focuses on the feature level by masking all the features of selected nodes (\dropnodef) with additional regularization on the outputs to maintain model robustness. In contrast, the ``dropping'' strategy of~\cite{dropnode_s}
operates at the structural level by removing the selected nodes from the original graph (\dropnodes). However, it requires normalization of the adjacency matrix in each layer due to dropped nodes, making it challenging to implement in models that rely on concatenating intermediate node representations or using \skipconnection. Both \dropnodef and \dropnodes are limited in terms of addressing over-smoothing without offering adequate flexibility. In contrast, our proposed \skipnode technique is simple and applicable to various models. Moreover, \skipnode is theoretically proven to effectively address all three issues (cf. Sec.~\ref{sec:skipnode}).}

\para{\skipnode \textit{vs.} \dropmessage.}
\lwgrevise{In the fifth column of Figure~\ref{fig:difference}, \dropmessage inherits the concept of ``dropping'' by masking the propagated features to node 2 to avoid over-smoothing, leaving the other two issues unresolved. The original experiments in the paper~\cite{dropmessage} only show the results for up to 8 layers, raising concerns about its performance in deeper architectures. Additionally, modifying the internal code of GNN models is required for implementing \dropmessage, which limits its general applicability. On the other hand, our \skipnode offers greater versatility as it only requires adjustments to the model's input and output without any modifications to the internal operations.}

\para{\skipnode \textit{vs.} \skipconnection.} 
\lwgrevise{\skipconnection \emph{adds} the inputs to the outputs of the current layer (shown in the sixth column of Figure~\ref{fig:difference}), while SkipNode \emph{replaces} the outputs of some nodes with the corresponding inputs. Due to this key difference, \skipconnection cannot effectively handle the performance degradation of deep GCNs, while \skipnode can. Specifically, the additive operation of \skipconnection fails to effectively mitigate over-smoothing since all the nodes still go through $L$ full convolutional operations. Consequently, the gradient vanishing issue predominantly occurs at the last layer due to zero outputs induced by over-smoothing (the zero output issue was proved in Proposition 3 from~\cite{oono2019graph}). Then, only a small number of gradients can be propagated backward through the shortcut paths established by \skipconnection. The mutual reinforcement between over-smoothing and gradient vanishing still quickly leads to performance degradation of deep GCNs. In comparison, \skipnode reduces the number of convolutional operations for the selected nodes by allowing them to skip some of these operations via our replacement operations. By doing so, \skipnode effectively reduces the risk of over-smoothing. Moreover, our \skipnode offers a simultaneous solution to the gradient vanishing issue. As the over-smoothing issue is alleviated, the gradients at the last layer increase. Consequently, gradients can swimmingly pass through the skipped nodes to shallow layers.}

%% file: chapters/preliminaries.tex
This section presents background contents, including vanilla GCN~\cite{gcn} and existing studies about over-smoothing. We first define the common notations used in this paper.

\subsection{Notations}
Assuming that $\mathbb N_{+}$ is the set of positive integers, we denote $[M] := \{1, \cdots, M\}$ for $M \in \mathbb N_{+}$. Let $\mathcal{G}=\{\mathcal{V}, \mathcal{E}\}$ denote an undirected graph, $\mathcal{V}=\{v_1, v_2, ..., v_N \}$ denote the node set, where $N$ is the number of all nodes. Let $\mathcal{E}=\{e_{ij}\}, i,j\in [N]$ denote the edge set, where $e_{ij}$ indicates the edge between nodes $v_i$ and $v_j$. The adjacency matrix is defined as $A \in \mathbb{R}^{N\times N}$, where $A_{ij} = 1$ if $e_{ij}\in \mathcal{E}$ otherwise $A_{ij} = 0$. Let $D$ denote the diagonal degree matrix, where $D_{ii} = \sum_{j=0}^{N}A_{ij}$. Suppose that $X^{(l)} \in \mathbb R^{N \times d_{l}}$ is the output feature matrix of the $l$-th layer, and $d_{l}$ is the feature dimensionality of the $l$-th layer, where $l \in [L]$ and $L \geq 2$. The raw node feature matrix is defined as $X^{(0)} \in \mathbb{R}^{N\times d_{0}}$. Given $v \in \mathbb R^{d}$ and $w \in \mathbb R^{d}$, the Kronecker product operation $v \otimes w$ results in a $\mathbb R^{d \times d}$ matrix, where $(v \otimes w)_{ij} := v_{i}w_{j}$. 

\subsection{Vanilla GCN}
The vanilla GCN~\cite{gcn} defines the graph convolution operation as $g_{\theta} \star x \approx \theta(I + D^{-\frac{1}{2}}AD^{-\frac{1}{2}})\boldsymbol{x}$, where $\boldsymbol{x} \in {{\mathbb R}^{N}}$ indicates the graph signal, $\star$ is the convolutional operation, and $\theta$ is a learnable parameter. To avoid exploding filter coefficients, ~\cite{gcn} further adopted a re-normalization trick. The final GCN formulation can thus be written as:
\begin{equation}
    X^{(l)} = ReLU(\tilde{A}X^{(l-1)}W^{(l)}),
    \label{eq:gcn}   
\end{equation}
where $\tilde{A} = (D + I)^{-\frac{1}{2}}(A + I)(D + I)^{-\frac{1}{2}}$ is a normalized adjacency matrix and $W^{(l)} \in \mathbb{R}^{d_{l-1} \times d_{l}}$ indicates the $l$-th learnable weight matrix. Accordingly, the symmetrical augmented graph Laplacian is defined as $\tilde{L} = I - \tilde{A}$.

\subsection{Over-smoothing of GCNs}
\label{subsec:over-smoothing-def}
Over-smoothing is an issue where node representations lose their diversities, which hinders the learning of graph neural networks. Many prior works~\cite{oono2019graph,jknet,cai2020note,appnp,li2018deeper} have been devoted to exploring what is over-smoothing and they can be roughly divided into two groups.  

On the one hand, from the perspective of linear feature propagation, \cite{li2018deeper,jknet,appnp} prove that GCNs can be regarded as low-pass filters retaining low-frequency signals in graph signal processing. Regardless of the weight matrix, the output features of nodes from a connected $\mathcal{G}$ become linearly dependent after multiple propagations. It can be explained by:

\begin{equation}
\lim_{k \to \infty}\tilde{A}^{k}X^{(0)} = \boldsymbol{\pi}\boldsymbol{\pi^{T}}X^{(0)},
\end{equation}
where $\pi_{i} = \frac{\sqrt{{D_{ii}} + 1}}{\sum_{v \in N}\sqrt{(D_{vv} + 1)}}$ and $i \in [N]$. This equation can be derived by diagonalization of $\tilde{A}$ and the fact that $\tilde{A}$'s eigenvalues lie in $(-1, 1]$. 

On the other hand, \cite{oono2019graph,cai2020note} investigate over-smoothing in the general graph convolutional networks architecture. \cite{oono2019graph} offers a more comprehensive understanding of over-smoothing by explaining it as the decreasing distance between the node feature matrix and a lower-information subspace $\mathcal{M} \subset \mathbb R^{N \times d}$ after multiple convolutions\footnote{\cite{oono2019graph} assumes all the output feature matrices share the same feature dimensionality $d$ to facilitate derivation.}. Assuming that the eigenvalues of $\tilde{A}$ are sorted in ascending order as: $\lambda_{1} \leq \cdots \leq \lambda_{N}$, we have $-1 < \lambda_{1} < \cdots < \lambda_{N-M}$ and $\lambda_{N - M + 1} = \cdots = \lambda_{N} = 1$ ($0 < M < N$). Let $U$ be a $M$-dimensional subspace of $\mathbb R^{N}$, which is the eigenspace corresponding to the smallest eigenvalue of augmented normalized graph Laplacian $\tilde{L}$ (that is the eigenspace associated with the eigenvalue 1 of $\tilde{A}$). $U$ is assumed to have the following properties~\cite{oono2019graph}. 

\begin{assumption}
\label{asp1}
$U$ has an orthonormal basis $(e_{m})_{m \in [M]}$ that consists of non-negative vectors.
\end{assumption} 

\begin{assumption} 
\label{asp2}
$U$ is invariant under $\tilde{A}$, i.e., if $u \in U$, then $\tilde{A}u \in U$.
\end{assumption} 

Based on Assumption~\ref{asp1}, we can construct the subspace $\mathcal{M}$ that only contains information about connected components and node degrees. That is, for any $X \in \mathcal{M}$, nodes with equal degrees from the same connected component have identical row vectors in $X$.

\begin{definition}
(Subspace $\mathcal{M}$) Let $\mathcal{M} := U \otimes \mathbb R^{d} = \{\sum_{m=1}^{M} e_{m} \otimes w_{m} | w_{m} \in \mathbb R^{d}\}$ be a subspace of $\mathbb R^{N \times d}$, where $(e_{m})_{m \in [M]}$ is the orthonormal basis of $U$.
\end{definition}
The distance between subspace $\mathcal{M}$ and $X$ is denoted as $\DMone{X} := inf\{\lVert X - H\rVert_{F}|H \in \mathcal{M}\}$, where $\lVert \cdot \rVert_{F}$ is the Frobenius norm of a matrix. Then, for a GCN model with $L$ layers defined by Eq.~(\ref{eq:gcn}), its final output $X^{(L)}$ will exponentially approach the subspace $\mathcal{M}$. The convergence can be written as:
\begin{equation}
    \label{eq:shrinking_distance}
    \DMone{X^{(L)}} \leq s_{L-1}\lambda \DMone{X^{(L-1)}} \leq \cdots \leq \prod_{l=1}^{L}(s_{l}\lambda)\DMone{X^{(0)}},
\end{equation}
where $s_{l}$ is the maximum singular value of $W^{(l)}$ and $\lambda := max_{n \in [N - M]} |\lambda_{n}|$. In addition, the Proposition 3 from~\cite{oono2019graph} also notes that $X^{(l)}$ will exponentially approach 0 as $l \to \infty$ if $s\lambda < 1$ and the operator norm of $\tilde{A}$ is no larger than $\lambda$, where $s := sup_{l \in [L]} s_{l}$.

However, they simply attribute the performance degradation to over-smoothing. In this paper, we show that there are two other issues responsible for the failure of deep GCNs, i.e., gradient vanishing and model weights over-decaying. 

%% file: chapters/motivation.tex
Numerous works~\cite{oono2019graph,jknet,cai2020note,appnp,li2018deeper} have explained the over-smoothing phenomenon as the sole reason that induces performance degradation. However, few methods have comprehensively analyzed the specific reasons for the degrading performance of deep GCNs beyond over-smoothing, i.e., gradient vanishing and weight over-decaying.

\newlength{\myImageWidth}
\setlength{\myImageWidth}{0.31\textwidth}

\begin{figure*}[ht]
        \centering
        \subfigure[\lwgrevise{Over-smoothing}]{ 
        \includegraphics[width=\myImageWidth]{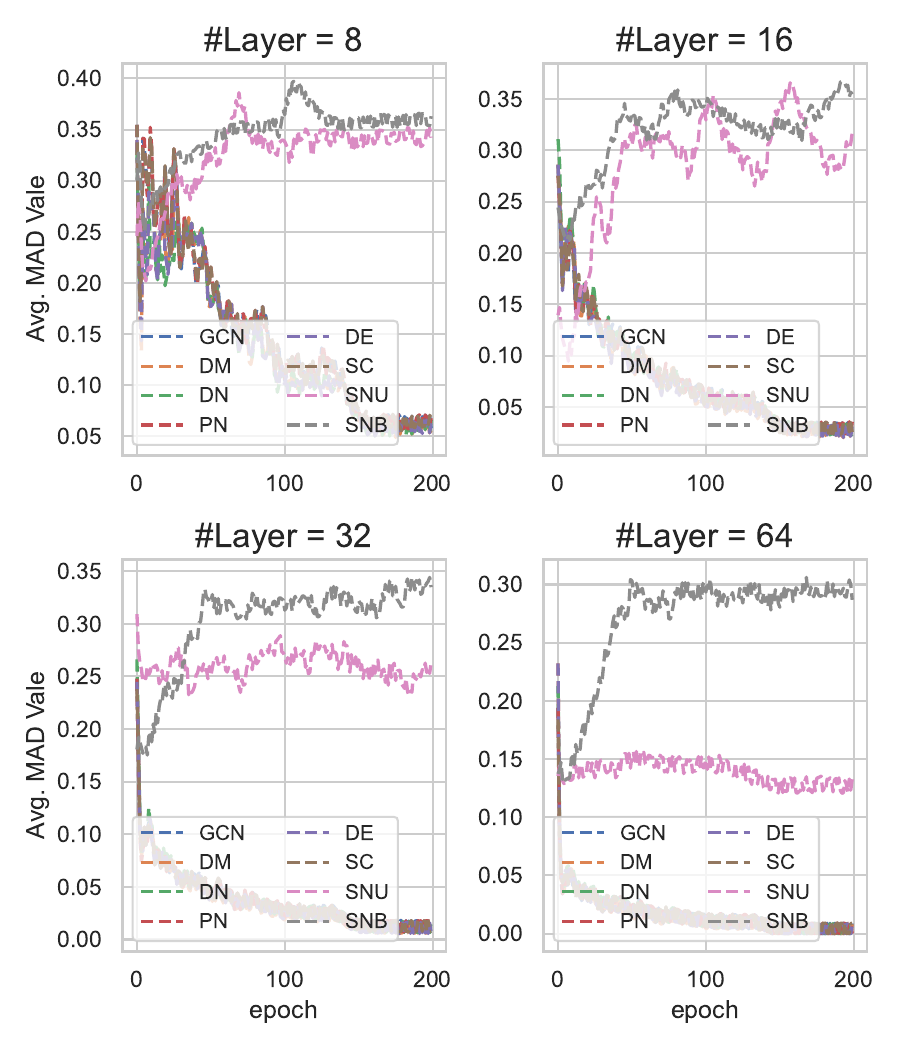}
        }
%        \hfil
        \subfigure[\lwgrevise{Gradient Vanishing}]{
        \includegraphics[width=\myImageWidth]{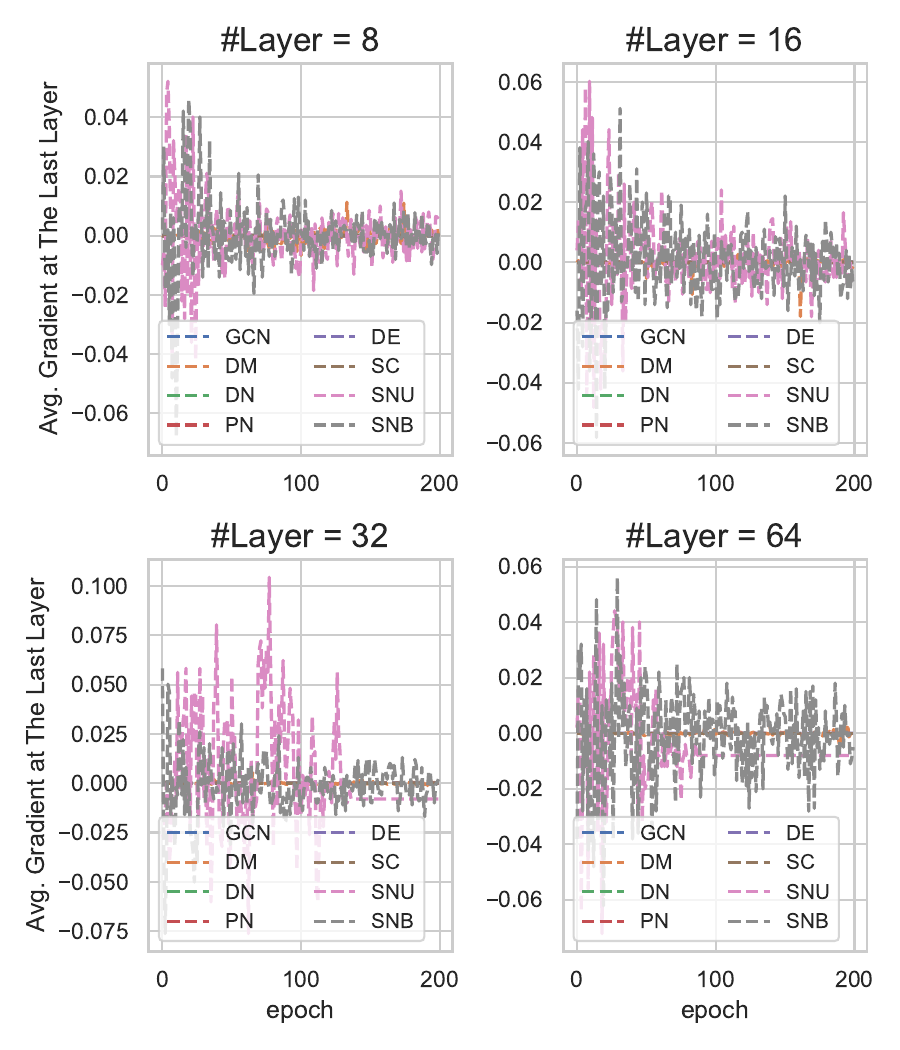}
        }
%        \hfil
        \subfigure[\lwgrevise{Model Weights Over-decaying}]{
        \includegraphics[width=\myImageWidth]{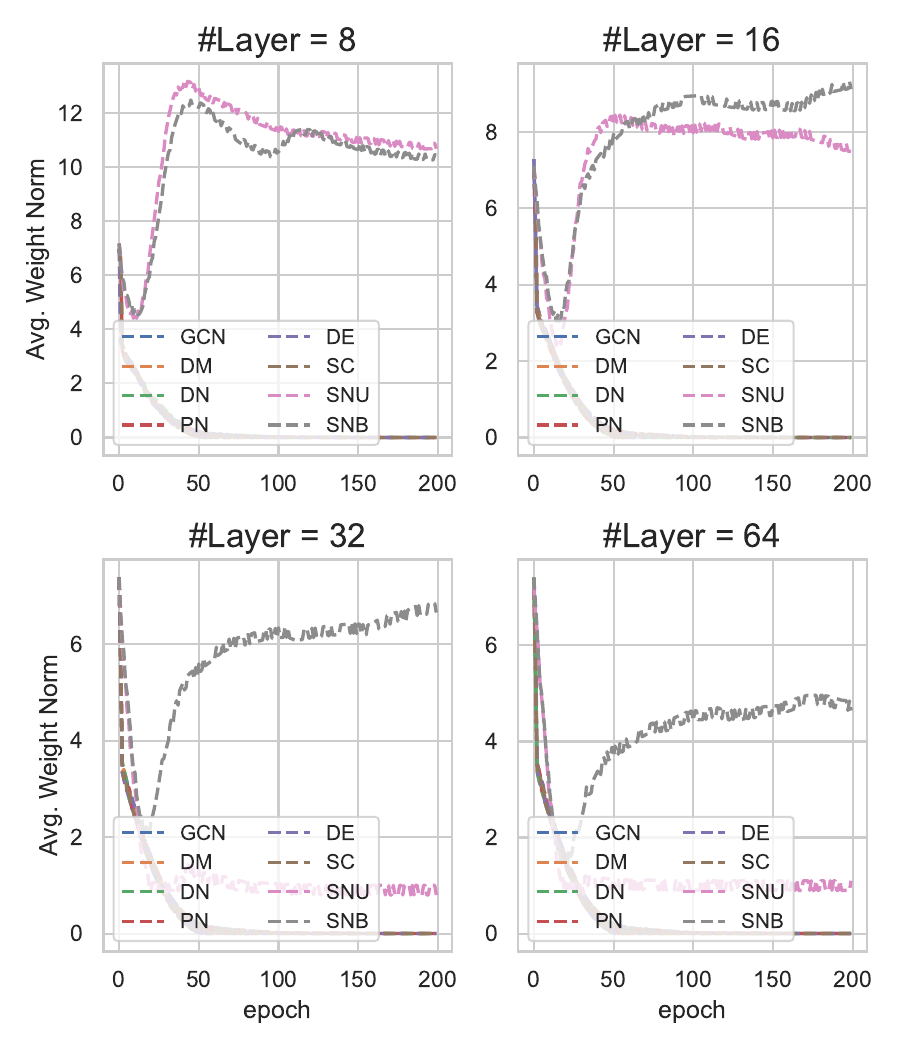}
        }
        \caption{\lwgrevise{Three issues on 8/16/32/64-layer GCNs with \pairnorm (PN), \dropedge (DE), \dropnode (DN), \dropmessage (DM), \skipconnection (SC), \skipnodeu (SNU), and \skipnodeb (SNB) using the Cora dataset. SNU/SNB refer to for \skipnode using uniform/biased sampling strategy. In (a), the features get more indistinguishable as the average MAD value from all the layers becomes smaller. In (b), only \skipnode alleviates the gradient vanishing issue. (c) shows the mean value of the L2 norm of all the model weights. Note that some curves overlap.}}
        \label{fig:negative_effects_on_cora}
\end{figure*}

\subsection{Gradient Vanishing}
\label{subsec:gradient-vanishing}
Gradient vanishing is a common issue existing in deep neural networks~\cite{gradient-vanishing}. However, there are very limited works analyzing the relationship between gradient vanishing and over-smoothing in detail. Here, we prove that the gradient would be reduced to zero at the classification layer as the model becomes deeper.

\begin{theorem}
\label{thm_gradient_vanishing}
Let $\mathcal{L}$ denote the cross-entropy loss function, $\mathcal{V}_{train}$ be the training set that contains $B$ nodes and $Z \in \mathbb{R}^{N \times C}$ be the output of the classification layer, where $C$ is the number of class. Assuming there are same samples in each category ($\frac{B}{C}$ samples per class), the gradient at the classification layer $\sum_{i \in \mathcal{V}_{train}} \sum_{j=1}^{C} \frac{\partial \mathcal{L}}{\partial Z_{ij}}$ gets close to 0 when the model's output converges to 0, the trivial fixed point introduced by the Proposition 3 from~\cite{oono2019graph}.
\end{theorem}

\begin{proof}
cf. Appendix~\ref{supp_sec:theorem1}. 
\end{proof}

Theorem~\ref{thm_gradient_vanishing} indicates if we generate class-balanced mini-batches (commonly used in most studies) to train a deep GCN model, the gradient at the output layer will be near 0 at the very beginning of the training process.

\subsection{Model Weights Over-decaying}
\label{subsec:weight-over-decaying}
When the gradient at the classification layer gets to 0, model weights over-decaying will occur if weight regularization (e.g., L2 loss) is used during training. Assume that regularization loss and classification loss are both included in the total loss. Theorem~\ref{thm_gradient_vanishing} points out that the over-smoothing issue causes the gradient to vanish, thus disabling the classification objective. As a result, the regularization objective plays a dominant role during training. The over-smoothing issue might be made worse by the regularization objective which decreases the norm of the weight matrices without control. Smaller weights make the output features more quickly approach zero.

\begin{remark}
\label{remark}
In deep GCNs, over-smoothing, gradient vanishing, and weights over-decaying together cause the performance degradation problem. Over-smoothing and gradient vanishing mutually exacerbate each other, leading to more serious performance degradation. At the beginning of the training process, the gradient at the classification layer is close to zero due to the near-zero output features caused by over-smoothing. The model weights would thus drastically decrease due to the regularization objective's dominating influence when the classification objective is disabled by the zero gradient. Over-smoothing would then become worse because the successive multiplication of the input by weight matrices with small values would more easily produce zero output features. Consequently, the above issues seriously hinder the training of deep GCNs. \lwgrevise{We further provide visualizations of the three issues using GCNs with $8/16/32/64$ layers on Cora in Figure~\ref{fig:negative_effects_on_cora}. We record the average mad value of node representations from all the layers, gradients at the last layer, and the average L2 norm of the model parameters. Besides, these results are derived from the averaging of 10 independent runs. It shows that most comparing strategies exhibit a mitigating effect only on over-smoothing, particularly in the case of a relatively shallow GCN (e.g., 8 layers). However, as we increase the depth, all the strategies other than SkipNode, become susceptible to all three issues (i.e., over-smoothing, gradient vanishing, and weight over-decaying). It's noteworthy that SkipConnection, while attempting to address gradient vanishing, encounters limitations in the deep GCNs due to serious over-smoothing. This results in the gradients at the last layer approaching zero, causing only a small number of gradients to propagate backward. By incorporating \skipnode using uniform sampling (\skipnodeu) and biased sampling (\skipnodeb), all the issues can be alleviated.}
\end{remark}

%% file: chapters/skipnode.tex
In this section, we present the methodological details of SkipNode and explain how SkipNode alleviates all three issues. 

\subsection{SkipNode}
\label{subsec:Methodology}
\begin{figure}[!bpht]
    \centering
    \includegraphics[width=0.8\columnwidth]{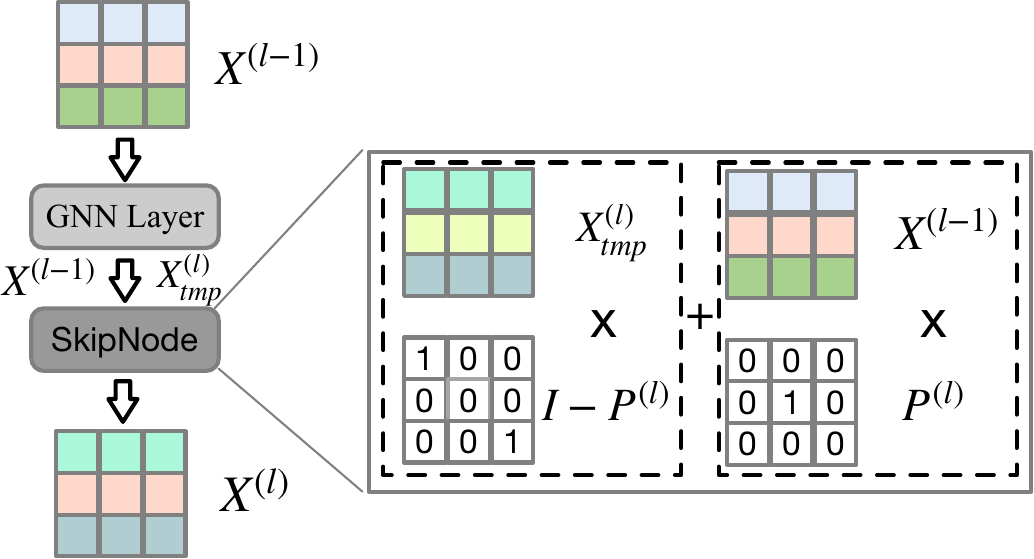}
    \caption{The forward propagation scheme of SkipNode. $X^{(l)}_{tmp} = \tilde{A}X^{(l-1)}W^{(l)}$. For simplicity, we omit the nonlinear function in this figure.}
    \label{fig:skip_procedure}
\end{figure}
Suppose a deep GCN contains $L$ layers. Then, for each middle layer, the \skipnode generates a mask matrix $P^{(l)} \in \mathbb{R}^{N \times N}$, which is a diagonal matrix whose diagonal consists of $1$ or $0$. $P_{ii} = 1$ indicates that node $i$ is a sampled node; otherwise, it is not. Then, the $l$-th GCN layer with \skipnode can be defined as follows:
\begin{equation}
\label{eq:skipnode}
X^{(l)} = \sigma\left((I - P^{(l)})\tilde{A}X^{(l-1)}W^{(l)} + P^{(l)}X^{(l-1)}\right),
\end{equation}
where $\sigma$ is a nonlinear function, i.e., ReLU function. The mask matrix $P^{(l)}$ is determined by the sampling strategy and sampling rate $\rho$. For those selected nodes, the ``skip'' operation suppresses the exponential smoothing effect of model depth by keeping their features unchanged in the current layer. \lwgrevise{In terms of outcomes, SkipNode does not require performing $L$ full convolutional operations for the selected nodes.} Here, we propose two sampling strategies: \textbf{Uniform Sampling} samples the mask matrix $P$ with $P_{ii} \sim Bernoulli(\rho)$; \textbf{Biased Sampling} samples $\rho N$ ``skipped nodes'' with sampling probabilities proportional to nodes' degrees because \cite{gcnii} points out that nodes with higher degrees are more likely to suffer from over-smoothing in deep GCNs.  We only employ \skipnode during training for each middle layer. Figure~\ref{fig:skip_procedure} describes the forward propagation scheme of \skipnode.

\begin{figure*}[!htbp]
        \centering
        \subfigure[Log distance ratio between each layer's output and the initial feature input.]{ 
        \includegraphics[width=0.48\textwidth]{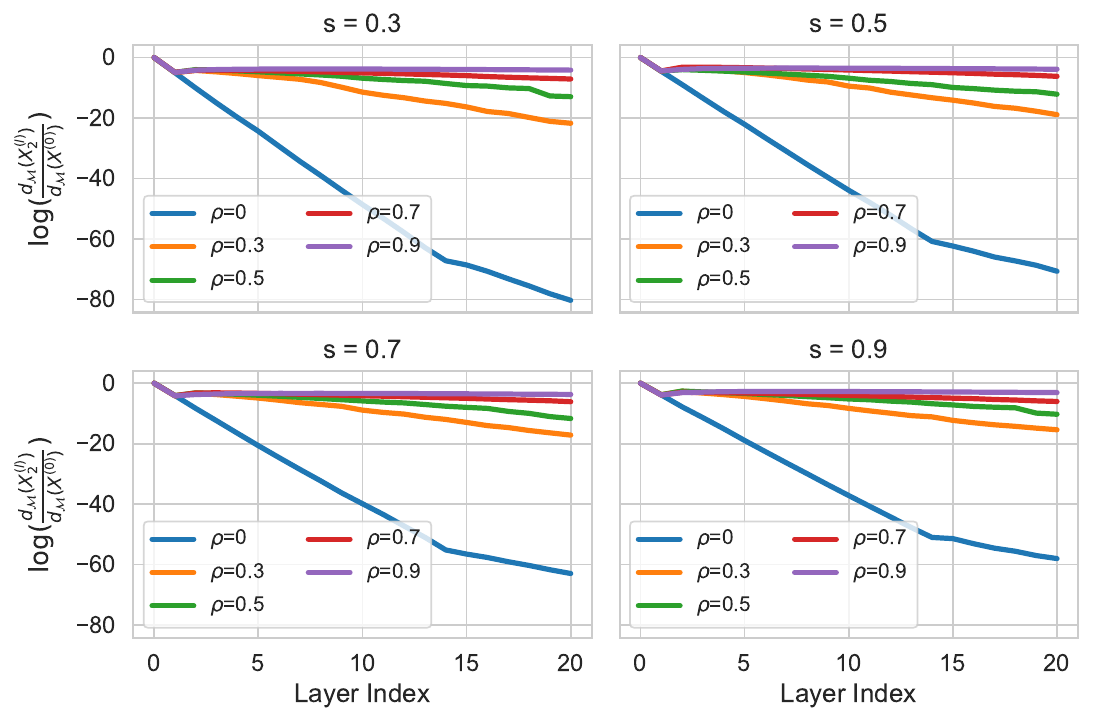}
        }
        \subfigure[Log distance ratio between one layer's output with/without \skipnode.]{
        \includegraphics[width=0.48\textwidth]{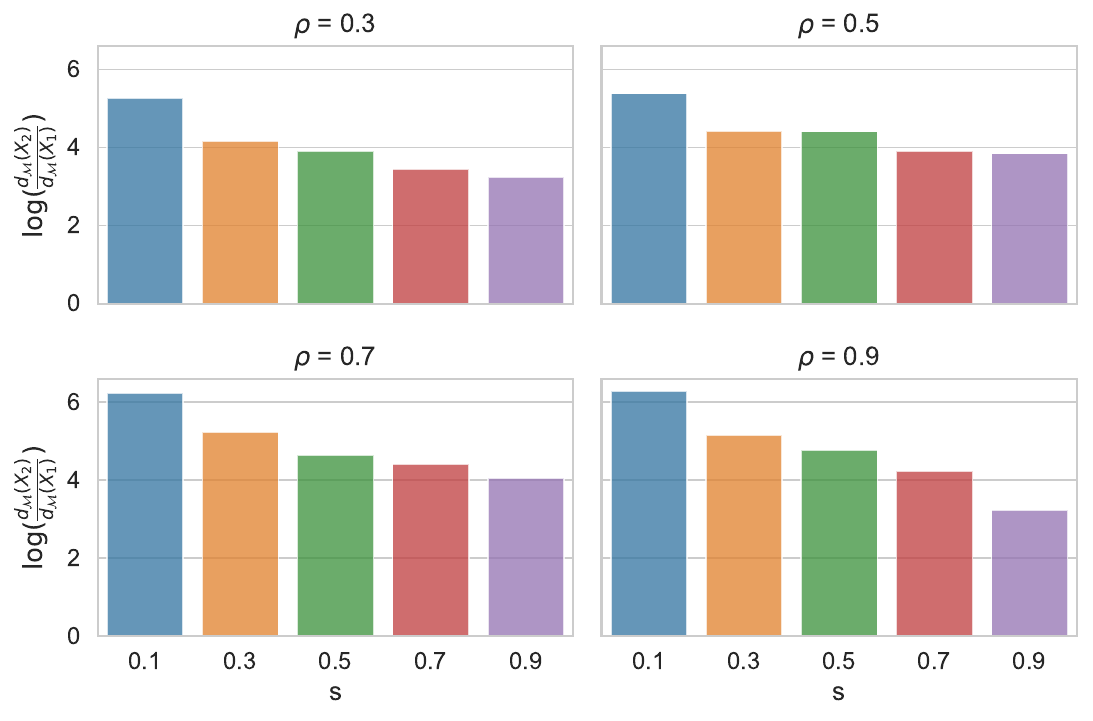}
        }
        \caption{The effectiveness of alleviating over-smoothing. We explore (a) the log distance ratio $\log(\frac{\DMone{X^{(l)}_{2}}}{\DMone{X^{(0)}}})$ for each layer $l$ with varying $\rho$ and $s$ ($\rho=0$ is equivalent to vanilla GCN), and (b) the log distance ratio $\log(\frac{\DMone{X_{2}}}{\DMone{X_{1}}})$ for one layer calculation. All the results are averaged from $100$ runs on the same graph for each fixed combination of $\rho$ and $s$.}
        \label{fig:allievat_over_smoothing}
\end{figure*}

\subsection{Toward Alleviating Performance Degradation}
Here, we provide theoretical and empirical analysis of \skipnode's ability\footnote{For simplicity, the analysis is based on uniform sampling.} to alleviate the over-smoothing issue (recall that over-smoothing is described by the \textbf{exponentially decreasing upper bound} of the distance between output features and the lower-information subspace $\mathcal{M}$, as shown in Eq.~(\ref{eq:shrinking_distance})). Then its beneficial impact on relieving gradient vanishing and weight over-decaying is discussed.

\subsubsection{Alleviating over-smoothing}
\skipnode can naturally alleviate the exponentially decreasing upper bound issue. This is because this issue is caused by performing too many graph convolutions, and \skipnode reduces the number of convolution operations for randomly selected nodes. In the next, we further theoretically analyze the anti-over-smoothing effects of \skipnode in expectation for \textit{one layer's calculation}. We first prove that \skipnode can, in comparison to vanilla GCN, improve the upper bound coefficient $(s_l \lambda)$ (Theorem~\ref{thm1}). We then demonstrate that the output of a layer with \skipnode is, in expectation, farther away from $\mathcal{M}$ than the output of a vanilla GCN layer (Theorem~\ref{thm2}).

For clarity, we focus on an arbitrary middle layer and omit the layer index $l$ in notations. Following \cite{oono2019graph}, both input and output dimensionalities are set to $d$. Let $X \in \mathbb{R}^{N \times d}$ be the input feature matrix. We denote the output of a GCN layer as $X_{1} = ReLU(\tilde{A}XW)$ and the output of a GCN layer with \skipnode as $X_{2} = (I - P)X_{1} + PX$ (since $X$ is the output of the previous layer and consequently only contains non-negative values, it is easy to verify that Eq.~(\ref{eq:skipnode}) can be written in this form). Since the layer index is removed, we abuse the notation $s$ a bit to denote it as the maximum singular value of the weight matrix $W$. Recall that $\lambda$ is the second largest magnitude of $\tilde{A}$'s eigenvalues, and $\rho \in (0, 1)$ is the sampling rate of \skipnode. The two theorems and a necessary lemma are presented as follows.

\begin{lemma}
\label{lem1}
For any $Z$, $Y \in \mathbb{R}^{N \times d}$, we have $\DMtwo{Z + Y} \leq \DMtwo{Z} + \DMtwo{Y} + 2\DMone{Z}\DMone{Y}$ and $\DMtwo{Z - Y} \geq \DMtwo{Z} + \DMtwo{Y} - 2\DMone{Z}\DMone{Y}$.
\end{lemma}

\begin{proof}[Proof of Lemma~\ref{lem1}]
\label{proof_of_lemma2}
cf. Appendix~\ref{supp_sec:lemma1}
\end{proof}

\begin{theorem}[Higher Upper Bound]
\label{thm1}
Under Assumptions~\ref{asp1} and \ref{asp2}, for any $\rho \in (0,1)$, we have $\DMone{\mathbb{E}[X_{2}]} \leq \left(s\lambda + \rho(1 - s\lambda)\right)\DMone{X}$ compared to $\DMone{X_{1}} \leq s\lambda\DMone{X}$ (cf. Theorem 1 of \cite{oono2019graph}). In particular, when $s\lambda < 1$, \skipnode can achieve a higher upper bound, i.e., $s\lambda + \rho(1 - s\lambda) > s\lambda$.
\end{theorem}

\begin{proof}[Proof of Theorem \ref{thm1}]
\label{proof1}
cf. Appendix~\ref{supp_sec:theorem2}
\end{proof}

\begin{theorem}[Longer Distance from $\mathcal{M}$]
\label{thm2}
Under Assumptions~\ref{asp1} and \ref{asp2}, when $\rho(\frac{1}{s\lambda} + 1) - 1 > 0$, \skipnode can guarantee a lower bound as $\DMone{\mathbb{E}[X_{2}]} \geq \left( \rho(\frac{1}{s\lambda} + 1) - 1\right)\DMone{X_{1}}$. In particular, when $\rho(\frac{1}{s\lambda} + 1) > 2$, the output of \skipnode is farther away from $\mathcal{M}$, i.e., $\frac{d_{\mathcal{M}}(\mathbb{E}[X_{2}])}{d_{\mathcal{M}}(X_{1})} > 1$.
\end{theorem}

\begin{proof}[Proof of Theorem~\ref{thm2}]
\label{proof2}
cf. Appendix~\ref{supp_sec:theorem3}
\end{proof}

\begin{remark}
Theorem~\ref{thm1} indicates that \skipnode can help GCN increase the upper bound of the distance between output features and $\mathcal{M}$ in expectation when $s\lambda < 1$. This bound is better than $\DMone{X_{1}} \leq s\lambda\DMone{X}$ in theory since the Cauchy-Schwarz inequality could be tight when $\{w_m\}$ and $\{w^{\prime}_{m}\}$ have a (nearly) scaling relationship. Theorem~\ref{thm2} underlines \skipnode can enhance GCN to obtain the output that is, in expectation, farther away from $\mathcal{M}$ in the case of $\rho(\frac{1}{s\lambda} + 1) > 2$. The conditions $s\lambda < 1$ (Theorem~\ref{thm1}) and $\rho(\frac{1}{s\lambda} + 1) > 2$ (Theorem~\ref{thm2}) could be easily satisfied in practice: $s < 1$ since the weight matrices are often initialized with small values and norm-constrained to avoid over-fitting; $\lambda < 1$ always holds. Taking the GCN model in the Cora dataset as a typical example, we have $\lambda \approx 0.996$ and $s \approx 0.2$, leading to $s\lambda \approx 0.199 < 1$. In this case, $\rho > 0.34$ suffices to secure $\DMone{\mathbb{E}[X_{2}]} \geq \DMone{X_{1}}$ in Theorem~\ref{thm2}. Besides, keeping $s$ fixed, the larger and denser the graph is ($\lambda$ decreases), the higher the upper bound in Theorem~\ref{thm1} becomes, and the wider the range of $\rho$ satisfying Theorem~\ref{thm2}'s condition goes.
\end{remark}

To further empirically validate the effectiveness of \skipnode, we generate an Erd\H{o}s -- R\'{e}nyi graph~\cite{erdos1959random,gilbert1959random} with $500$ nodes and edge generation probability $0.5$. Then, we vary $s$ and $\rho$ to investigate the performance of \skipnode for alleviating over-smoothing. We conduct $100$ runs for each $(\rho,s)$ combination where the input features and weight matrices are initialized randomly for each run. The averaged results are shown in Figure~\ref{fig:allievat_over_smoothing}. Figure~\ref{fig:allievat_over_smoothing}(a) shows that the log distance ratio $ \log(\frac{\DMone{X^{(l)}_{2}}}{\DMone{X^{(0)}}})$ of \skipnode ($\rho > 0$) is vastly larger than that of vanilla GCN ($\rho = 0$) at all the middle layers. This should be because \skipnode can improve both the exponent and base parts of the convergence speed coefficient $(s\lambda)^L$. Moreover, it demonstrates that the convergence speed decreases with $\rho$ increasing, which is consistent with the upper bound analysis in Theorem~\ref{thm1}, i.e., the base part of the coefficient $(s\lambda)^L$ is increased by $\rho(1-s\lambda)$. Figure~\ref{fig:allievat_over_smoothing}(b) explores the log distance ratio $ \log(\frac{\DMone{X_{2}}}{\DMone{X_{1}}})$ for one layer calculation. We can see the log ratio is always larger than 0 (that is $\log 1$), and on average $\DMone{X_{2}}$ is around 50 times ($e^4$) farther away from $\mathcal{M}$ than $\DMone{X_{1}}$. Figure~\ref{fig:allievat_over_smoothing}(b) also indicates that increasing $\rho$ or decreasing $s$ can expand the gap between $\DMone{X_{2}}$ and $\DMone{X_{1}}$, which is consistent with the obtained lower bound of the ratio in expectation in Theorem~\ref{thm2}, i.e., $\rho(\frac{1}{s\lambda} + 1) - 1$. 

%And it also demonstrates that the convergence speed decreases with the sampling rate $\rho$ increasing which is also consistent with Theorem~\ref{thm1}. In Figure~\ref{fig:allievat_over_smoothing}(b), we can observe that the curves of the log distance ratio $ log\frac{\DMone{X^{(l)}_{2}}}{\DMone{X^{(l)}_{1}}}$ w.r.t $l$ are stable rising for various $s$. Besides, both theoretical and empirical results indicate that larger $\rho$ makes the increase more sharp.

The above analysis supports the effectiveness of the ``skipping'' operation for alleviating over-smoothing. Unlike uniform sampling which samples nodes with equal probability, biased sampling tends to choose high-degree nodes which suffer more from being smoothed. It is analogous to using different $\rho$ for different $P_{ii}$ according to $v_{i}$'s degree. Therefore, we could still maintain the effectiveness of \skipnode with biased sampling. 

\subsubsection{Alleviating gradient vanishing and weight over-decaying}
Here we explain why \skipnode can alleviate the back-propagation-induced and over-smoothing-induced gradient vanishing issues. Assume that the graph signal $\boldsymbol{x}$ is transformed by $\theta_{l}, l \in [L]$ and we have $\boldsymbol{x}^{(l)} = \theta_{l}\boldsymbol{x}^{(l-1)}$ (we omit the feature propagation step since transformation is the main cause of back-propagation-induced gradient vanishing). Therefore, the gradient from the classification loss function $f(\boldsymbol{x}^{(L)})$ to $\boldsymbol{x}^{(0)}$ is:
\begin{small}
    \begin{align*}
    \frac{\partial f(\boldsymbol{x}^{(L)})}{\partial \boldsymbol{x}^{(0)}} &= \frac{\partial f(\boldsymbol{x}^{(L)})}{\partial \boldsymbol{x}^{(L)}} \frac{\partial \boldsymbol{x}^{(L)}}{\partial \boldsymbol{x}^{(L-1)}}\cdots\frac{\partial \boldsymbol{x}^{(1)}}{\partial \boldsymbol{x}^{(0)}}\\
    &=  (\prod_{l=1}^{L}\theta_{l})\frac{\partial f(\boldsymbol{x}^{(L)})}{\partial \boldsymbol{x}^{(L)}}.
\end{align*}
\end{small}
If $|\theta_{l}| <1 $, the successive multiplication of $\theta_{l}$ will converge to zero, leading to (nearly) zero gradient. \skipnode can reduce the times of successive multiplication by letting randomly selected nodes skip several transformations to slow down the convergence speed. The over-smoothing-induced gradient vanishing issue (Theorem~\ref{thm_gradient_vanishing}) is caused by (nearly) zero output which is in turn caused by over-smoothing (cf. Proposition 3 of~\cite{oono2019graph}). \skipnode can also effectively relieve this issue by skipping several layers' calculations in the forward phase of GCNs.

As we have analyzed in Sec.~\ref{subsec:weight-over-decaying}, the model weights over-decaying issue is triggered by gradient vanishing when regularization is used. Since \skipnode can alleviate both over-smoothing and gradient vanishing, the model weights are thus updated by the gradient passed from the classification layer. In contrast to vanilla deep GCNs, \skipnode can keep model weights at reasonable levels.

%% file: chapters/experiment.tex
\input{tables/dataset.tex}

In this section, we evaluate \skipnode on various graphs and tasks to \lwgrevise{answer the following research questions:
\begin{itemize}
	\item [\textbf{Q1:}] Is \skipnode effective in alleviating the performance degeneration for deep GCNs? (cf. Sec.~\ref{sec:q1})
	\item [\textbf{Q2:}] Is \skipnode general to be used on different graphs? (cf. Sec.~\ref{sec:q2})
	\item [\textbf{Q3:}] Is \skipnode scalable for handling large-scale graphs? (cf. Sec.~\ref{sec:q3})
	\item [\textbf{Q4:}] Is \skipnode more efficient compared to other strategies? (cf. Sec.~\ref{sec:q4})
	\item [\textbf{Q5:}] What's the role of the large sampling rate $\rho$ in \skipnode? (cf. Sec.~\ref{sec:q5})
\end{itemize}
}
We use \textbf{\skipnodeu} and \textbf{\skipnodeb} to denote \skipnode\footnote{\url{https://github.com/WeigangLu/SkipNode}} with uniform sampling and biased sampling, respectively. \lwgrevise{Without specification, we conduct our experiments using a single NVIDIA TITAN RTX GPU with 24 GB of memory and an Intel(R) Core(TM) i9-10980XE CPU with 62 GB of main memory.}

\input{chapters/exp_intro.tex}

\input{tables/semi.tex}

\subsection{\lwgrevise{Effectiveness Study (Q1)}}
\label{sec:q1}
\lwgrevise{
In this section, we evaluate the effectiveness of our \skipnode by applying it to 6 GCN models with different layers within the range of $\{4, 8, 16, 32, 64\}$, which is the standard layer selection as established in~\cite{gcnii}, under the semi-supervised node classification task.}

\para{\lwgrevise{Backbone Models and Configurations.}} 
\lwgrevise{
We select GCN~\cite{gcn}, GAT~\cite{gat}, GraphSAGE~\cite{sage}, JKNet~\cite{jknet}, InceptGCN~\cite{inceptgcn}, and GCNII~\cite{gcnii} as backbone models. We first perform a grid search for each backbone model and select the hyperparameters that yield the best accuracy on the validation set of each benchmark. Except for their suggested configurations, we conduct a search for the dropout rate in $\{0, 0.05, 0.1, \ldots, 0.8\}$, the weight decay rate in $\{5e-4, 5e-7, 5e-9\}$, and the learning rate in $\{0.01, 0.05, 0.1\}$. The hidden dimensionality is set to 64.}

\para{\lwgrevise{Strategies and Configurations.}}
\lwgrevise{We employ several comparing strategies, including \texttt{DropEdge}~\cite{dropedge}, \texttt{DropNode-F}~\cite{dropnode_f}, \texttt{PairNorm}~\cite{pairnorm}, and \texttt{SkipConnect\-ion}~\cite{residual}. For each backbone model with its the best hyperparameters, we apply all the comparing strategies while only tuning the strategy-related hyperparameters and the dropout rate. The other hyperparameters are kept fixed. Regarding \skipnode, \dropedge, and \dropnodef, we explore their sampling rates in $\{0, 0.05, 0.1, \ldots, 0.8, 0.9\}$. For \pairnorm, we directly utilize the suggested hyperparameters provided in its code.}

\para{\lwgrevise{Performance Comparison.}}
\lwgrevise{
Table~\ref{tab:semi} illustrates the impact of increasing the number of layers ($L$) on the performance of various models across the three datasets. It is evident that, for models designed for shallow architectures (such as GCN, GAT, and GraphSAGE), increasing $L$ leads to significant performance declines. Among the different strategies, namely \dropedge, \pairnorm, and \skipconnection, it is observed that they can only maintain satisfactory performance when used in shallow architectures with $L = 4$ or 8. However, as the number of layers increases to 16 or higher, significant performance drops become apparent. Conversely, the \dropnodef strategy consistently exhibits inferior performance across all cases. This can be attributed to the all-feature removing operation which easily induces zero outputs. It is worth noting that certain methods display zero standard deviations in the results (e.g., the 64-layer GCN with \dropedge on the Cora dataset). This occurs when the final outputs are all zero, resulting in trivial predictions. This observation aligns with the analysis presented in Remark~\ref{remark}. Our \skipnode consistently improves the performance of all models with various $L$. For shallow-architecture models such as GCN, GAT, and GraphSAGE, our \skipnodeb tends to achieve better performance compared to \skipnodeu. It is because higher-degree nodes are more susceptible to over-smoothing~\cite{gcnii}. For deep-architecture models such as JKNet, InceptGCN, and GCNII, which exhibit a lower susceptibility to over-smoothing, the improvement produced by \skipnodeu also stems from its anti-overfitting effect, as \skipnodeu serves as valuable data augmentation for each layer. Unlike \dropedge that randomly removes potentially valuable edges, \skipnodeu only manipulates the learned representations. For \dropnodes which requires an odd model depth and \dropmessage which modifies the inner message-passing operation of different GCN models, we specifically compare our \skipnode against them in Table~\ref{tab:comp_all_strategys}. Our \skipnode still emerges as the most effective approach in alleviating performance degeneration. It consistently outperforms all other strategies, achieving the highest accuracy. Note that most methods suffer from accuracy loss when we increase the depth. This is because we follow previous work~\cite{gcnii} to choose much higher depth values than usual settings (typically 2\textasciitilde{}3) of GNNs. This ensures that the impact of over-smoothing is obvious so that the research conclusions hold more pervasively. An example for demonstrating \skipnode's usefulness: the 2-layer GraphSAGE achieves 67.2\% on Citeseer, while the performance of its 4-layer counterpart decreases to 64.1\%. In contrast, the 4-layer GraphSAGE equipped with \skipnodeu can reach 68.5\%.
}

\input{tables/add_comp.tex}

\input{tables/full.tex}

\subsection{\lwgrevise{Generalizability Study (Q2)}}
\label{sec:q2}
\lwgrevise{
In this section, we assess the generalizability of \skipnode by employing it on eight GCNs across seven graphs. The graphs include both homophilic and heterophilic graphs, and our evaluation focuses on the full-supervised node classification task.
 }
 
\para{\lwgrevise{Backbone Models and Configurations.}}
\lwgrevise{We use eight backbone models, including GCN, GAT, JKNet, InceptGCN, GCNII, GRAND~\cite{dropnode_f}, GPRGNN~\cite{gprgnn}, and APPNP~\cite{appnp}. We follow the same approach for hyperparameter search as described in Sec.~\ref{sec:q1}.
}

\para{\lwgrevise{Strategies and Configurations.}}
\lwgrevise{We apply the same comparing strategies and hyperparameter search rule as described in Sec.~\ref{sec:q1}.
}

\para{\lwgrevise{Performance Comparison.}}
\lwgrevise{
Table~\ref{table:full} presents a summary of the performance of different GCN architectures using various strategies for full-supervised node classification. Among the strategies examined, namely \dropedge, \pairnorm, and \skipconnection, only minor improvements are observed across different GCNs. Besides, \dropedge even induces performance declines in some small-scale graphs (e.g., Texas and Wisconsin), where the limited edges might hold valuable information for classification tasks. On the other hand, the effectiveness of \dropnodef as a plug-and-play strategy is limited since it removes the features of all the selected nodes, which can be detrimental to the performance of the model. In contrast, our proposed strategy, \skipnode, consistently achieves significant improvements across various GCNs and all graphs, demonstrating its remarkable generalizability.
}

\begin{figure}[!h]
        \centering
        \subfigure[Node Classification]{ 
        \includegraphics[width=0.46\columnwidth]{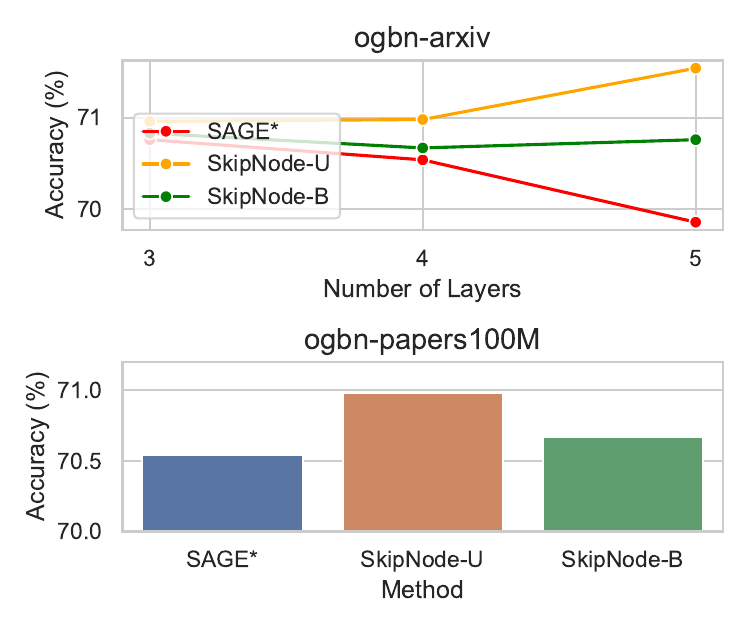}
        }
        \subfigure[Link Prediction]{
        \includegraphics[width=0.46\columnwidth]{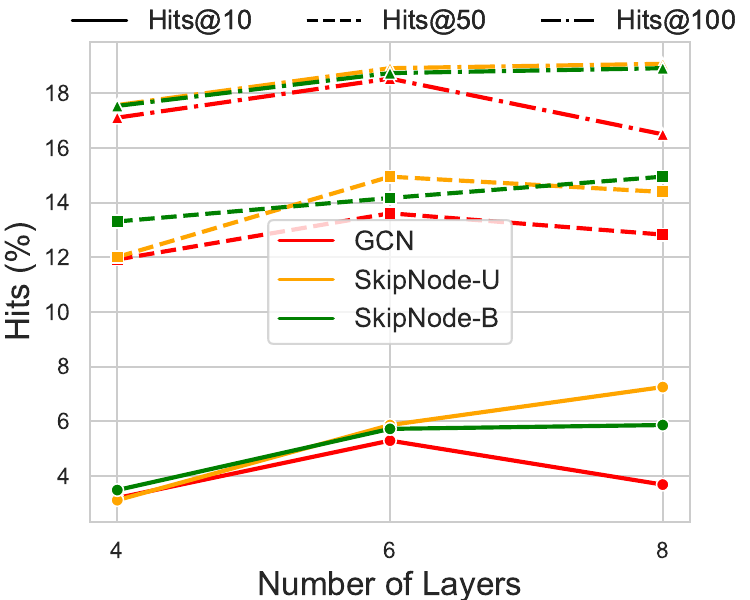}
        }
        \caption{\lwgrevise{Scalability analysis. We use three large-scale graphs under different tasks, including (a) node classification task using ogbn-arxiv (upper), ogbn-papers100M (bottom), and (b) link prediction task using ogbl-ppa. SAGE$*$ is the abbreviation for GraphSAGE\_res\_incep.}}
        \label{fig:large}
\end{figure}

\subsection{\lwgrevise{Scalability Study (Q3)}}
\label{sec:q3}
\lwgrevise{
In this section, we explore the scalability of \skipnode to three large-scale graphs under two different tasks, including node classification and link prediction tasks.
}

\para{\lwgrevise{Hardware.}}
\lwgrevise{We conduct our experiments using a single Teas V100 GPU with 32GB of memory and an Intel Xeon Gold 8163 CPU with 24 cores.}

\para{\lwgrevise{Backbone Model and Configuration.}}
\lwgrevise{For the node classification task, we follow the publicly available implementation\footnote{\url{https://github.com/mengyangniu/ogbn-papers100m-sage}} on the OGB Leaderboards\footnote{\url{https://ogb.stanford.edu/docs/leader_nodeprop/#ogbn-papers100M}} and utilize GraphSAGE\_res\_incep as our backbone model. GraphSAGE\_res\_incep is a variant of GraphSAGE, which is commonly used on the large-scale graph. It is built in an inception-like structure, where the hidden representations from different receptive fields are concatenated. And each convolutional layer is equipped with SkipConnection. The concatenated representations are then passed through a two-layer MLP for classification. We keep the hidden dimensionality fixed at 1024. For the link prediction task, we use GCN as the backbone model.
}

\begin{figure*}[!htbp]
        \centering
        \subfigure[\lwgrevise{Cora}]{ 
        \includegraphics[width=\myImageWidth]{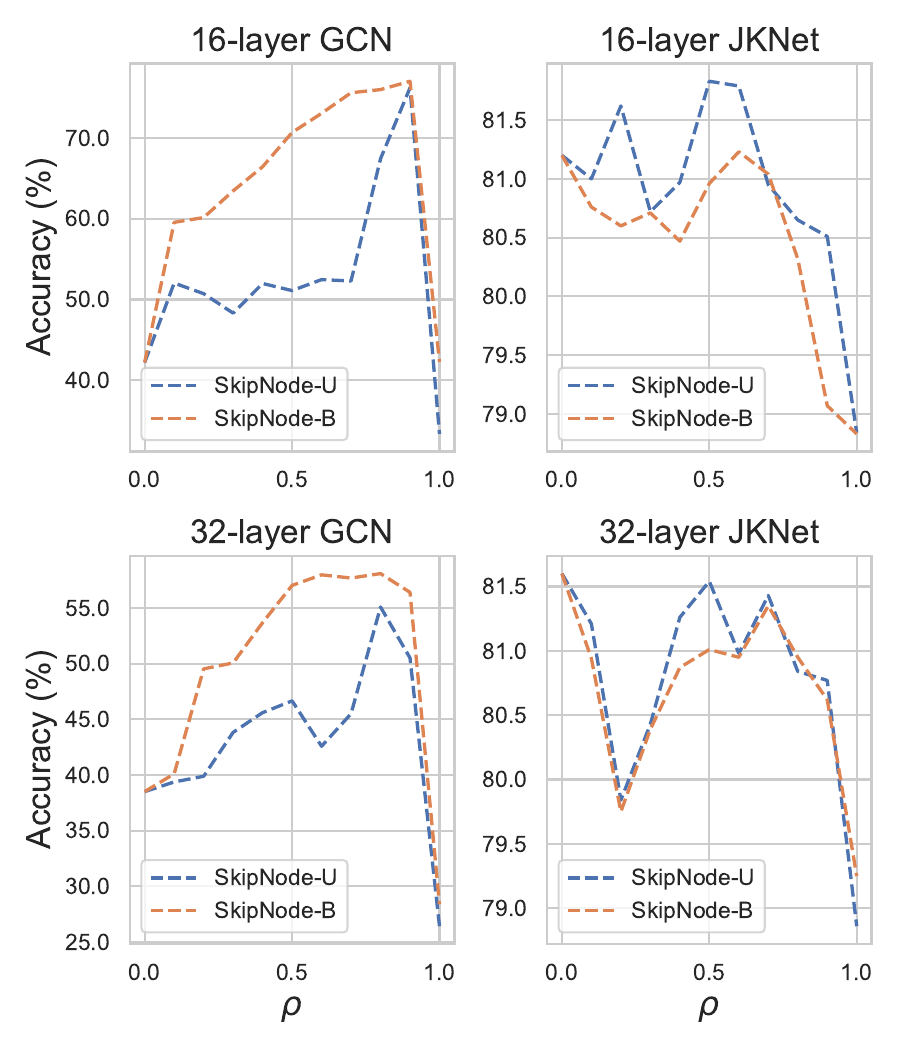}
        }
%        \hfil
        \subfigure[\lwgrevise{Citeseer}]{
        \includegraphics[width=\myImageWidth]{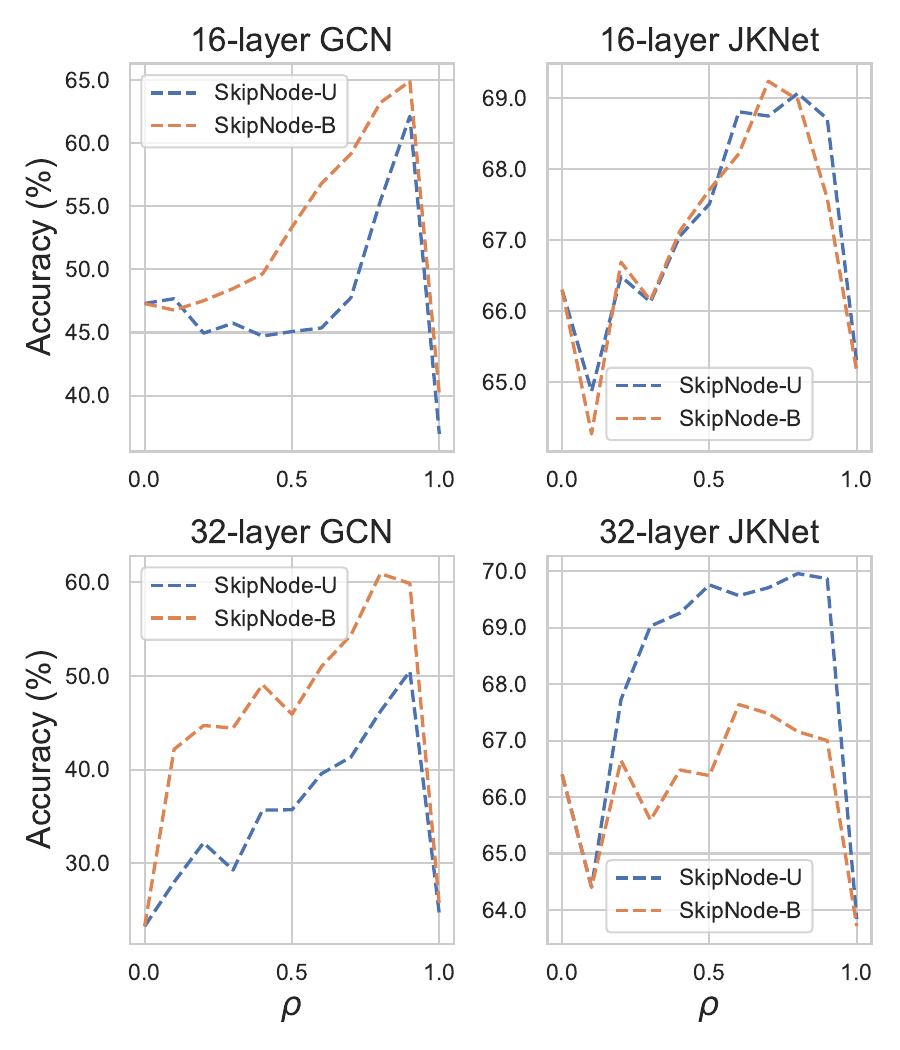}
        }
%        \hfil
        \subfigure[\lwgrevise{Pubmed}]{
        \includegraphics[width=\myImageWidth]{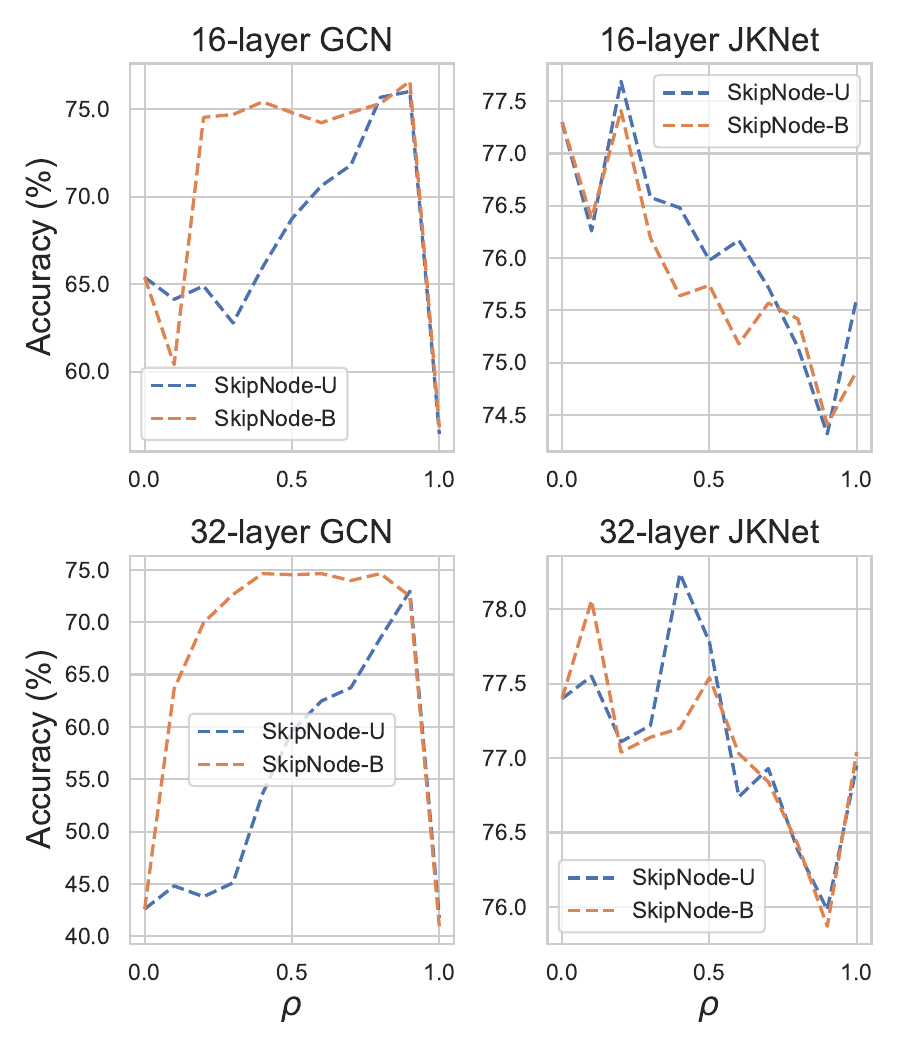}
        }
        \caption{\lwgrevise{Hyperparameter study of $\rho$. We use 16- and 32-layer GCN and JKNet as the backbone models on Cora, Citeseer, and Pubmed datasets. All the results are averaged from 10 runs.}}
        \label{fig:hyper}
\end{figure*}

\begin{figure}[!htbp]
\centering
\includegraphics[width=0.9\columnwidth]{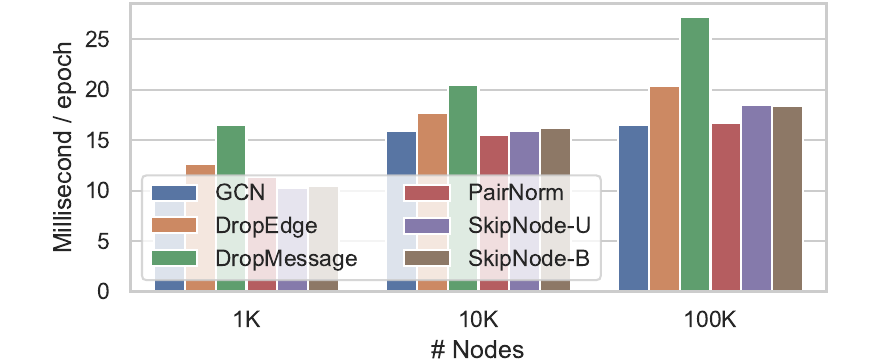}
\caption{\lwgrevise{Efficiency analysis. We record the wall-clock time per epoch on randomly generated graphs.}}
\label{fig:efficiency}
\end{figure}

\para{\lwgrevise{Performance Comparison.}}
\lwgrevise{In Figure~\ref{fig:large} (a), we present the evaluation of our \skipnode approach in the node classification task. For the ogbn-arxiv dataset (upper), we vary the number of layers in $\{3, 4, 5\}$. Remarkably, we observe that incorporating \skipnode into GraphSAGE\_res\_incep benefits from deeper architectures, leading to improved performance. For the ogbn-papers100M dataset (bottom), we fix the number of layers at 4 due to the computational resource constraints. Our \skipnode still showcases performance improvements. Besides, \skipnode only incurs a modest increase in running time, with 8.1\% and 10.7\% increases compared to the vanilla model (22297s, 22823s, and 20612s per epoch for \skipnodeu, \skipnodeb, and the vanilla model, respectively). In Figure~\ref{fig:large} (b), we find that \skipnode proves useful in the link prediction task when deepening GCN. This finding can be attributed to the fact that the quality of node representations plays a vital role in the success of this task.
}

\subsection{\lwgrevise{Efficiency Study (Q4)}}
\label{sec:q4}
\lwgrevise{
In this section, we present the mean training time per epoch for 100 epochs on randomly generated graphs, measured in seconds of wall-clock time. The computational cost of a 3-layer GCN and its variant with \dropedge, \dropmessage, \pairnorm, and \skipnode is illustrated in Figure~\ref{fig:efficiency}. It can be observed that \dropedge and \dropmessage impose additional computational burden on training GCN. This is due to the required normalization of \dropedge after each sampling operation and the dropping operation over each aggregated message of \dropmessage. On the other hand, \skipnode and \pairnorm demonstrate better efficiency as they only manipulate the node representations without involving additional costly operations. 
}

\subsection{\lwgrevise{Hyperparameter Study (Q5)}}
\label{sec:q5}
\lwgrevise{
In this section, we investigate the sensitivity of the sampling rate parameter $\rho$, in our \skipnode framework. The range of $\rho$ values considered in our experiments is $\{0, 0.1, 0.2, \cdots, 0.9, 1.0\}$. Each trial consists of 500 epochs, and we conduct 10 independent runs to ensure robustness. Our evaluation is performed on the Cora, Citeseer, and Pubmed datasets using the GCN and JKNet models, both with 16 and 32 layers.}

\lwgrevise{
The results are presented in Figure~\ref{fig:hyper}. Notably, we observe that the performance of the \skipnode is influenced by the anti-over-smoothing ability of the baseline model. For models more susceptible to over-smoothing, such as GCN, a larger sampling rate (0.8 or 0.9) leads to improved performance. This is because higher sampling rates allow more nodes to avoid being over-smoothed during propagation. Conversely, for models that are less affected by over-smoothing, such as JKNet, a moderate sampling rate ranging from 0.3 to 0.6 results in better performance.}

\lwgrevise{
It is worth noting that when $\rho = 1$, we observe significant performance drops in all the cases. This can be attributed to the fact that if all nodes ``skip'' the intermediate layers, the parameters of these layers are inadequately optimized, leading to degenerated performance.
}

%Besides, we further explore the smoothness of the learned features (output from the second-to-last layer) by MAD value using a 32-layer GCN as the backbone model. In Figure~\ref{fig:hyper-mad}, we show the MAD value of the learned features after $500$ training epochs. We can see, on each dataset, the MAD value returned by GCN is always zero, indicating its poor feature diversity. In comparison, \skipnode successfully maintains information diversities as feature propagating.
%
%
%
%\begin{figure}[ht]
%	\centering
%	\includegraphics[width=0.6\columnwidth]{fig/fig_smoothness.pdf}
%    \caption{Hyperparameter study on $\rho$ in term of MAD Value.}
%    \label{fig:hyper-mad}
%\end{figure}

%\begin{figure*}[ht]
%        \centering
%        \subfigure[Accuracy]{               
%        \includegraphics[width=0.58\linewidth]{fig/fig_rho.pdf}
%        }
%        \hspace{0.2in}
%        \subfigure[MAD value]{
%        \includegraphics[width=0.2\linewidth]{fig/fig_smoothness.pdf}
%        }
%        \caption{Hyperparameter study on $\rho$ in term of accuracy and MAD value. We use GCN, GCN (\skipnodeu), and GCN (\skipnodeb) with 32 layers.}
%        \label{fig:rho_analysis}
%\end{figure*}

%% file: tables/dataset.tex
\begin{table}[ht]
    \caption{\lwgrevise{Dataset Statistics.}}
    \label{tab:data_sta}
    \centering
    \resizebox{1.0\linewidth}{!}{
    \begin{tabular}{c l c c c c}
        \toprule
        Category & Dataset &  \#Nodes & \#Edges & \#Features\\
        \midrule
        \multirow{3}{*}{Homo.} & Cora & $2,708$ & $5,429$ & $1,433$\\ 
        ~ & Citeseer & $3,327$ & $4,732$ & $3,703$\\
        ~ & Pubmed & $19,717$ & $44,338$ & $500$\\
        \midrule
        \multirow{4}{*}{Hete.} & Chameleon & $2,277$ & $36,101$ & $2,325$\\
        ~ & Cornell & $183$ & $295$ & $1,703$\\
        ~ & Texas & $183$ & $309$ & $1,703$\\
        ~ & Wisconsin & $251$ & $499$ & $1,703$\\
        \midrule
        \multirow{3}{*}{\lwgrevise{Large.}} 
        & \lwgrevise{ogbn-papers100M} & \lwgrevise{$111,059,956$} & \lwgrevise{$1,615,685,872$} & \lwgrevise{$128$} \\
        & ogbn-arxiv & $169,343$ & $1,166,243$ & $128$ \\
        ~ & ogbl-ppa & $ 576,289$ & $ 30,326,273$ & $58$\\
        \bottomrule
    \end{tabular}
    }
\end{table}

%% file: chapters/exp_intro.tex
\subsection{Datasets}
We use 9 graph datasets, and their detailed statistics are described in Table~\ref{tab:data_sta}. These datasets fall into three categories as follows.

\para{Homophilic Graphs.}
\textbf{Cora, Citeseer, and Pubmed}\footnote{\url{https://linqs.soe.ucsc.edu/data}} are citation networks in which each publication is defined by a 0/1-valued word vector indicating the presence or absence of the corresponding dictionary word.

\para{Heterophilic Graphs.}
\textbf{Chameleon}\footnote{\url{http://snap.stanford.edu/data/wikipedia-article-networks.html}} is a network based on the English Wikipedia, representing pages dedicated to the topic of chameleon. Articles are represented by nodes, and mutual links between them are represented by edges. \textbf{Cornell}, \textbf{Texas}, and \textbf{Wisconsin}\footnote{\url{http://www.cs.cmu.edu/afs/cs.cmu.edu/project/theo-11/www/wwkb/}} are subdatasets from the webpage dataset collected from computer science departments of various universities. Each page (node) is represented by the bag-of-words representation and can be classified into one of five categories. 
%The edges between nodes represent web page hyperlinks.

\para{Large-scale Graphs.}
\textbf{ogbn-papers100M and ogbn-arxiv}\footnote{\url{https://github.com/snap-stanford/ogb}} are both paper citation networks extracted from the Microsoft Academic Graph. The former contains 111 million papers while the latter only includes the Computer Science arxiv papers. A 128-dimensional feature vector is generated for each paper (node) in the dataset by averaging the embeddings of words in the title and abstract. \textbf{ogbl-ppa}\footnotemark[7] is a protein-protein association network. Nodes represent proteins from 58 different species, and edges indicate biological associations between proteins.

\subsection{Tasks}
We evaluate our \skipnode on various graph datasets through different tasks as follows:

\textbullet\
\textbf{Semi-supervised node classification task:} We use three widely-used graphs, including Cora, Citeseer, and Pubmed under public splits~\cite{yang2016revisiting} (20 nodes per class as training data, 500/1000 nodes as validation/test data).

\textbullet\ 
\textbf{Full-supervised node classification task:} We use Cora, Citeseer, Pubmed, Chameleon, Cornell, Texas, and Wisconsin under ten different splits ($60\%$, $20\%$ and $20\%$ for training, validation, and test). 

\textbullet\
\textbf{Node classification task on a large-scale graph:} We use two challenging datasets, i.e., ogbn-arxiv and ogbn-papers100M. We abide by the dataset splits in the GitHub repository\footnotemark[7].

\textbullet\ 
\textbf{Link prediction task}: We evaluate \skipnode on ogbl-ppa and the evaluation metric is the ratio of predicted positive edges that are ranked at the K-th place or above (Hits@K).

For all the numerical results, the best performance w.r.t each backbone model (for each value of layer numbers) is highlighted in \textbf{bold} and the second best is \underline{underlined}.

%% file: tables/semi.tex
\begin{table*}[!ht]
    \centering
    \caption{\centering{\lwgrevise{Effectiveness Analysis: Semi-supervised Classification Accuracy $\pm$ Standard Deviation (\%) w.r.t Different Depths.}}}
    \resizebox{1.0\textwidth}{!}{
    \begin{tabular}{p{2.5cm} | p{0.95cm} p{0.95cm} p{0.95cm} p{0.95cm} p{0.95cm} | p{0.95cm} p{0.95cm} p{0.95cm} p{0.95cm} p{0.95cm} | p{0.95cm} p{0.95cm} p{0.95cm} p{0.95cm} p{0.95cm}}
    \toprule
    GNN Arch.  & \multicolumn{5}{c|}{Cora} & \multicolumn{5}{c|}{Citeseer} & \multicolumn{5}{c}{Pubmed} \\

    ~+\texttt{Strategy} 
    &\multicolumn{1}{c}{\centering $L$ = 4} &\multicolumn{1}{c}{\centering $L$ = 8} &\multicolumn{1}{c}{\centering $L$ = 16} &\multicolumn{1}{c}{\centering $L$ = 32} &\multicolumn{1}{c|}{\centering $L$ = 64} 
    &\multicolumn{1}{c}{\centering $L$ = 4} &\multicolumn{1}{c}{\centering $L$ = 8} &\multicolumn{1}{c}{\centering $L$ = 16} &\multicolumn{1}{c}{\centering $L$ = 32} &\multicolumn{1}{c|}{\centering $L$ = 64} 
    &\multicolumn{1}{c}{\centering $L$ = 4} &\multicolumn{1}{c}{\centering $L$ = 8} &\multicolumn{1}{c}{\centering $L$ = 16} &\multicolumn{1}{c}{\centering $L$ = 32} &\multicolumn{1}{c}{\centering $L$ = 64} \\
    \midrule
    GCN & $81.2_{\pm0.7}$ & $78.9_{\pm2.7}$ & $42.2_{\pm9.7}$ & $38.5_{\pm2.8}$ & $31.9_{\pm1.8}$ & $66.5_{\pm1.8}$ & $57.8_{\pm2.2}$ & $47.3_{\pm6.5}$ & $23.3_{\pm5.3}$ & $23.6_{\pm0.9}$ & $77.3_{\pm0.6}$ & $76.6_{\pm0.8}$ & $65.4_{\pm6.7}$ & $42.6_{\pm3.0}$ & $42.0_{\pm2.7}$ \\

    ~ +\dropedge
    & $\underline{81.6_{\pm0.6}}$ & $79.0_{\pm1.3}$ & $44.8_{\pm6.1}$ & $31.9_{\pm0.1}$ & $31.9_{\pm0.0}$ 
    & $67.5_{\pm2.1}$ & $58.9_{\pm1.7}$ & $48.6_{\pm6.3}$ & $23.2_{\pm0.3}$ & $24.4_{\pm0.6}$ &
     $\textbf{77.9}_{\pm0.8}$ & $\underline{77.0_{\pm0.8}}$ & $68.9_{\pm4.1}$ & $40.0_{\pm6.9}$ & $41.2_{\pm1.3}$ \\    
     
    ~ +\dropnodef
    &  $79.9_{\pm1.1}$ & $77.4_{\pm7.7}$ & $37.2_{\pm6.2}$ & $27.9_{\pm6.2}$ & $31.9_{\pm0.0}$ 
    & $64.8_{\pm2.2}$ & $56.0_{\pm4.6}$ & $38.6_{\pm1.7}$ & $20.5_{\pm1.7}$ & $20.1_{\pm0.2}$ 
    & $76.8_{\pm1.0}$ & $76.4_{\pm0.7}$ & $60.5_{\pm4.3}$ & $31.5_{\pm1.2}$ & $18.0_{\pm0.0}$ \\
    
    ~ +\pairnorm
    &  $81.4_{\pm1.1}$ & $78.7_{\pm2.2}$ & $41.0_{\pm5.1}$ & $30.3_{\pm5.1}$ & $31.9_{\pm0.0}$ 
    & $67.8_{\pm1.5}$ & $60.1_{\pm1.9}$ & $49.0_{\pm5.5}$ & $23.1_{\pm0.0}$ & $27.2_{\pm0.5}$ 
    & $77.5_{\pm0.7}$ & $76.4_{\pm1.3}$ & $64.5_{\pm4.7}$ & $39.5_{\pm6.0}$ & $42.1_{\pm1.7}$ \\
    
    ~ +\skipconnection
    &  $81.5_{\pm0.8}$ & $78.7_{\pm1.7}$ & $40.3_{\pm4.2}$ & $26.8_{\pm7.7}$ & $31.9_{\pm0.0}$ 
    & $66.9_{\pm1.1}$ & $58.5_{\pm2.2}$ & $45.8_{\pm4.8}$ & $22.6_{\pm1.5}$ & $23.2_{\pm0.1}$ 
    & $77.2_{\pm0.8}$ & $76.1_{\pm0.8}$ & $62.7_{\pm2.3}$ & $37.2_{\pm9.7}$ & $40.6_{\pm0.6}$\\
    
    ~ +\textbf{\skipnodeu}
    &  $\underline{81.6_{\pm0.8}}$ & $\underline{79.1_{\pm1.3}}$ & $\underline{76.2_{\pm1.2}}$ & $\underline{55.1_{\pm2.7}}$ & $\underline{43.2_{\pm2.0}}$ 
    & $\textbf{69.6}_{\pm0.5}$ & $\textbf{67.2}_{\pm1.1}$ & $\underline{62.1_{\pm2.2}}$ & $\underline{50.5_{\pm1.6}}$ & $\underline{39.6_{\pm3.0}}$ 
    & $77.4_{\pm0.5}$ & $76.8_{\pm0.7}$ & $\underline{76.1_{\pm0.8}}$ & $\underline{72.9_{\pm1.4}}$ & $\underline{48.6_{\pm0.9}}$ \\

    ~ +\textbf{\skipnodeb}
    & $\textbf{82.0}_{\pm0.7}$ & $\textbf{80.4}_{\pm0.8}$ & $\textbf{78.1}_{\pm0.3}$ & $\textbf{58.1}_{\pm0.9}$ & $\textbf{43.4}_{\pm1.1}$ 
    & $\underline{68.8_{\pm1.2}}$ & $\underline{67.0_{\pm1.4}}$ & $\textbf{64.9}_{\pm0.8}$ & $\textbf{60.9}_{\pm1.4}$ & $\textbf{42.1}_{\pm1.2}$ 
    & $\underline{77.5_{\pm0.7}}$ & $\textbf{77.3}_{\pm0.5}$ & $\textbf{76.6}_{\pm0.6}$ & $\textbf{74.6}_{\pm1.0}$ & $\textbf{50.8}_{\pm0.8}$ \\
    
    \midrule

    GAT & $80.4_{\pm0.9}$ & $74.4_{\pm1.6}$ & $31.9_{\pm0.1}$ & $31.9_{\pm0.0}$ & $31.9_{\pm0.0}$ & $68.1_{\pm1.4}$ & $56.9_{\pm2.3}$ & $41.1_{\pm6.0}$ & $25.9_{\pm4.1}$ & $20.8_{\pm4.3}$ & $77.0_{\pm0.7}$ & $73.5_{\pm2.7}$ & $43.0_{\pm3.4}$ & $41.2_{\pm0.9}$ & $41.4_{\pm1.8}$ \\

    ~ +\dropedge
    & $\underline{81.2_{\pm0.9}}$ & $73.1_{\pm2.5}$ & $31.9_{\pm0.0}$ & $31.9_{\pm0.0}$ & $31.9_{\pm0.0}$ & $\underline{68.2_{\pm1.8}}$ & $56.7_{\pm2.8}$ & $42.1_{\pm8.7}$ & $22.6_{\pm4.8}$ & $20.8_{\pm4.3}$ & $77.5_{\pm0.7}$ & $\textbf{76.0}_{\pm0.7}$ & $41.3_{\pm4.6}$ & $\underline{41.9_{\pm3.1}}$ & $40.0_{\pm2.7}$ \\

    ~ +\dropnodef
    & $78.8_{\pm1.0}$ & $64.9_{\pm11.8}$ & $26.6_{\pm4.6}$ & $31.9_{\pm0.0}$ & $31.9_{\pm0.0}$ & $65.4_{\pm1.9}$ & $52.5_{\pm5.5}$ & $36.6_{\pm2.1}$ & $21.5_{\pm2.6}$ & $21.8_{\pm4.0}$ & $76.1_{\pm0.8}$ & $74.3_{\pm1.9}$ & $36.0_{\pm9.0}$ & $36.5_{\pm5.3}$ & $18.0_{\pm0.0}$ \\
    
    ~ +\pairnorm
    & $80.8_{\pm1.1}$ & $71.9_{\pm3.2}$ & $29.1_{\pm5.1}$ & $31.9_{\pm0.0}$ & $31.9_{\pm0.0}$ & $68.1_{\pm1.4}$ & $56.8_{\pm2.6}$ & $43.1_{\pm6.7}$ & $22.1_{\pm3.7}$ & $19.8_{\pm3.4}$ & $77.2_{\pm1.1}$ & $73.7_{\pm1.4}$ & $38.9_{\pm7.0}$  & $40.8_{\pm0.5}$ & $38.6_{\pm0.8}$\\
    
    ~ +\skipconnection
    & $79.8_{\pm1.1}$ & $69.4_{\pm3.7}$ & $28.0_{\pm5.4}$ & $31.9_{\pm0.0}$ & $31.9_{\pm0.0}$ & $68.1_{\pm1.4}$ & $55.4_{\pm2.5}$ & $42.7_{\pm4.6}$ & $20.4_{\pm3.6}$ & $20.9_{\pm4.3}$ & $76.9_{\pm1.1}$ & $73.5_{\pm2.1}$ & $38.9_{\pm7.0}$ & $40.8_{\pm0.2}$ & $41.3_{\pm0.4}$\\
    
    ~ +\textbf{\skipnodeu}
    & $\textbf{81.6}_{\pm0.9}$ & $\textbf{79.1}_{\pm2.5}$ & $\underline{46.7_{\pm7.0}}$ & $\underline{43.8_{\pm9.2}}$ & $\underline{43.5_{\pm7.3}}$ 
    & $\underline{68.2_{\pm1.3}}$ & $\underline{65.8_{\pm1.4}}$ & $\underline{62.4_{\pm1.8}}$ & $\underline{32.8_{\pm5.1}}$ & $\underline{37.5_{\pm2.1}}$ 
    & $\underline{77.6_{\pm1.3}}$ & $\underline{75.7_{\pm1.0}}$ & $\underline{64.4_{\pm4.3}}$ & $\underline{41.9_{\pm2.1}}$ & $\underline{46.6_{\pm2.8}}$ \\

    ~ +\textbf{\skipnodeb}
    & $\underline{81.2_{\pm0.5}}$ & $\underline{77.8_{\pm1.3}}$ & $\textbf{79.3}_{\pm8.3}$ & $\textbf{59.9}_{\pm1.9}$ 
    & $\textbf{61.4}_{\pm2.8}$ & $\textbf{68.4}_{\pm1.1}$ & $\textbf{66.9}_{\pm1.8}$ & $\textbf{63.3}_{\pm1.0}$ & $\textbf{58.7}_{\pm1.4}$ & $\textbf{53.4}_{\pm1.9}$ 
    & $\textbf{78.0}_{\pm1.0}$ & $\textbf{76.0}_{\pm1.7}$ & $\textbf{69.6}_{\pm1.4}$ & $\textbf{56.0}_{\pm5.8}$ & $\textbf{52.4}_{\pm3.5}$ \\

    \midrule

    GraphSAGE & $80.6_{\pm1.3}$ & $73.8_{\pm2.1}$ & $57.2_{\pm4.7}$ & $31.9_{\pm0.0}$ & $31.9_{\pm0.1}$ & $64.1_{\pm1.8}$ & $54.1_{\pm3.9}$ & $43.0_{\pm9.9}$ & $22.0_{\pm5.4}$ & $19.1_{\pm1.8}$ & $76.6_{\pm0.9}$ & $73.2_{\pm1.9}$ & $64.6_{\pm4.7}$ & $40.7_{\pm0.0}$ & $40.7_{\pm0.0}$ \\

    ~ +\dropedge
    & $79.7_{\pm0.9}$ & $74.5_{\pm2.0}$ & $57.6_{\pm11.0}$ & $31.9_{\pm0.0}$ & $31.9_{\pm0.0}$ & $64.8_{\pm2.4}$ & $53.5_{\pm3.6}$ & $49.5_{\pm5.8}$ & $33.1_{\pm1.6}$ & $22.8_{\pm2.0}$ & $77.0_{\pm0.8}$ & $74.1_{\pm1.0}$ & $65.0_{\pm6.5}$ & $40.7_{\pm0.0}$ & $40.8_{\pm0.2}$ \\

    ~ +\dropnodef
    &  $77.9_{\pm0.8}$ & $72.0_{\pm2.2}$ & $50.3_{\pm13.1}$ & $27.7_{\pm4.0}$ & $28.2_{\pm3.3}$ & $62.0_{\pm1.7}$ & $53.0_{\pm3.3}$ & $35.7_{\pm11.3}$ & $20.3_{\pm1.8}$ & $21.3_{\pm1.9}$ & $75.6_{\pm0.7}$ & $71.6_{\pm3.7}$ & $57.1_{\pm9.3}$ & $40.9_{\pm2.0}$ & $40.7_{\pm2.3}$ \\
    
    ~ +\pairnorm
    & $79.2_{\pm0.9}$ & $70.8_{\pm2.8}$ & $52.8_{\pm9.4}$ & $31.9_{\pm0.0}$ & $31.9_{\pm0.0}$ & $64.1_{\pm1.9}$ & $54.3_{\pm4.1}$ & $46.7_{\pm6.3}$ & $20.1_{\pm5.0}$ & $19.1_{\pm1.8}$ & $76.6_{\pm0.9}$ & $74.0_{\pm1.1}$ & $66.8_{\pm4.4}$ & $40.7_{\pm0.0}$ & $40.7_{\pm0.0}$ \\
    
    ~ +\skipconnection
    & $79.7_{\pm0.8}$ & $71.1_{\pm2.7}$ & $45.7_{\pm12.0}$ & $31.9_{\pm0.0}$ & $31.9_{\pm0.0}$ & $63.6_{\pm1.5}$ & $51.7_{\pm6.6}$ & $42.4_{\pm8.7}$ & $18.3_{\pm0.3}$ & $19.1_{\pm1.8}$ & $75.9_{\pm1.1}$ & $73.7_{\pm2.2}$ & $63.7_{\pm4.1}$ & $40.7_{\pm0.0}$ & $40.7_{\pm0.0}$ \\
        
    ~ +\textbf{\skipnodeu}
    & $\textbf{81.5}_{\pm0.4}$ & $\underline{78.8_{\pm1.1}}$ & $\underline{66.8_{\pm1.0}}$ & $\underline{63.6_{\pm4.7}}$ & $\underline{41.8_{\pm3.2}}$ & $\textbf{68.5}_{\pm1.2}$ & $\underline{65.6_{\pm1.5}}$ & $\underline{62.0_{\pm1.4}}$ & $\underline{51.9_{\pm3.1}}$ & $\underline{28.9_{\pm2.7}}$ & $\textbf{77.4}_{\pm1.4}$ & $\underline{75.7_{\pm0.5}}$ & $\underline{74.7_{\pm0.7}}$ & $\underline{54.1_{\pm1.7}}$ & $\underline{41.1_{\pm1.8}}$ \\

    ~ +\textbf{\skipnodeb}
    & $\underline{81.4_{\pm0.5}}$ & $\textbf{79.2}_{\pm0.7}$ & $\textbf{77.5}_{\pm1.0}$ & $\textbf{68.2}_{\pm3.9}$ & $\textbf{54.3}_{\pm2.6}$ & $\underline{67.9_{\pm0.8}}$ & $\textbf{66.9}_{\pm0.9}$ & $\textbf{64.0}_{\pm1.1}$ & $\textbf{60.5}_{\pm1.5}$ & $\textbf{59.4}_{\pm1.7}$ & $\underline{77.3_{\pm0.8}}$ & $\textbf{76.0}_{\pm0.3}$ & $\textbf{74.8}_{\pm0.5}$ & $\textbf{58.0}_{\pm0.7}$ & $\textbf{53.3}_{\pm0.9}$ \\

    \midrule
    
    JKNet 
    & $81.1_{\pm0.9}$ & $\underline{81.8_{\pm0.6}}$ & $\underline{81.2_{\pm1.0}}$ & $\textbf{81.6}_{\pm1.2}$ & $\underline{80.5_{\pm1.0}}$ 
    & $67.5_{\pm1.2}$ & $67.4_{\pm0.8}$ & $66.3_{\pm1.0}$ & $66.4_{\pm1.3}$ & $\underline{66.1_{\pm1.6}}$ 
    & $77.0_{\pm0.7}$ & $\underline{77.1_{\pm0.8}}$ & $77.3_{\pm0.6}$ & $77.4_{\pm0.7}$ & $77.6_{\pm0.4}$\\
    
    ~ +\dropedge 
    & $81.3_{\pm1.3}$ & $81.5_{\pm1.6}$ & $81.1_{\pm1.1}$ & $81.3_{\pm1.1}$ & $79.7_{\pm1.6}$ 
    & $68.2_{\pm1.4}$ & $68.4_{\pm2.7}$ & $68.0_{\pm1.9}$ & $67.3_{\pm2.5}$ & $65.5_{\pm1.3}$ 
    & $77.3_{\pm1.1}$ & $76.8_{\pm1.1}$ & $77.1_{\pm0.9}$ & $77.6_{\pm1.3}$ & $77.7_{\pm1.5}$ \\
    
    ~ +\dropnodef
    & $79.0_{\pm1.2}$ & $76.0_{\pm1.3}$ & $71.8_{\pm1.5}$ & $72.5_{\pm1.0}$ & $73.4_{\pm1.5}$ 
    & $65.9_{\pm2.0}$ & $63.0_{\pm2.2}$ & $63.3_{\pm1.1}$ & $60.5_{\pm1.3}$ & $57.8_{\pm1.3}$ 
    & $74.1_{\pm1.2}$ & $74.8_{\pm1.6}$ & $74.6_{\pm0.9}$ & $74.3_{\pm0.9}$ & $73.7_{\pm0.9}$ \\
    
    ~ +\pairnorm
    & $81.5_{\pm1.2}$ & $81.4_{\pm0.9}$ & $81.1_{\pm1.4}$ & $81.0_{\pm1.2}$ & $79.4_{\pm1.2}$ 
    & $68.0_{\pm1.9}$ & $68.0_{\pm2.0}$ & $68.5_{\pm1.8}$ & $67.4_{\pm2.6}$ & $65.4_{\pm1.1}$ 
    & $77.5_{\pm1.1}$ & $\underline{76.9_{\pm1.6}}$ & $\underline{77.4_{\pm0.8}}$ & $77.9_{\pm1.4}$ & $77.9_{\pm0.8}$ \\
    
    ~ +\skipconnection
    & $80.6_{\pm1.7}$ & $73.4_{\pm1.3}$ & $73.4_{\pm1.4}$ & $73.3_{\pm1.4}$ & $72.9_{\pm1.9}$ 
    & $67.2_{\pm2.2}$ & $67.5_{\pm0.9}$ & $67.9_{\pm1.5}$ & $67.3_{\pm1.7}$ & $65.7_{\pm1.5}$ 
    & $77.1_{\pm1.1}$ & $76.3_{\pm1.4}$ & $76.2_{\pm1.3}$ & $76.8_{\pm1.3}$ & $77.7_{\pm0.9}$ \\
    
    ~ +\textbf{\skipnodeu}
    & $\textbf{81.7}_{\pm0.9}$ & $\textbf{82.1}_{\pm1.2}$ & $\textbf{81.8}_{\pm0.9}$ & $\underline{81.5_{\pm1.1}}$ & $\textbf{80.9}_{\pm1.1}$ 
    & $\textbf{70.0}_{\pm0.9}$ & $\underline{68.8_{\pm1.2}}$ & $\underline{69.0_{\pm0.8}}$ & $\textbf{69.0}_{\pm1.1}$ & $\textbf{67.6}_{\pm1.3}$ 
    & $\textbf{78.3}_{\pm0.7}$ & $\textbf{77.2}_{\pm0.7}$ & $\textbf{77.6}_{\pm1.1}$ & $\textbf{78.2}_{\pm0.9}$ & $\textbf{78.9}_{\pm0.4}$ \\

    ~ +\textbf{\skipnodeb}
    & $\underline{81.6_{\pm0.5}}$ & $81.6_{\pm1.1}$ & $\underline{81.2_{\pm0.4}}$ & $81.3_{\pm0.9}$ & $80.4_{\pm1.1}$ 
    & $\underline{69.3_{\pm1.0}}$ & $\textbf{69.4}_{\pm1.1}$ & $\textbf{69.2}_{\pm0.6}$ & $\underline{67.6_{\pm1.1}}$ & $\underline{66.1_{\pm1.4}}$ 
    & $\underline{78.0_{\pm0.7}}$ & $\underline{77.1_{\pm1.0}}$ & $\underline{77.4_{\pm0.8}}$ & $\underline{78.0_{\pm0.8}}$ & $\underline{78.2_{\pm0.3}}$\\
    
    \midrule
    
   InceptGCN 
    & $79.2_{\pm0.7}$ & $79.7_{\pm0.6}$ & $79.8_{\pm0.7}$ & $80.0_{\pm0.4}$ & $\underline{80.2_{\pm0.6}}$ 
    & $69.9_{\pm0.8}$ & $70.0_{\pm0.9}$ & $\underline{69.7_{\pm1.3}}$ & $69.8_{\pm0.6}$ & $69.4_{\pm0.7}$ 
    & $77.0_{\pm0.9}$ & $76.5_{\pm0.8}$ & $76.6_{\pm0.6}$ & $76.6_{\pm0.5}$ & $76.5_{\pm0.4}$\\
    
    ~ +\dropedge 
    & $79.0_{\pm1.1}$ & $79.8_{\pm1.3}$ & $80.1_{\pm0.9}$ & $79.8_{\pm0.5}$ & $79.6_{\pm0.4}$ 
    & $69.6_{\pm1.0}$ & $70.2_{\pm0.8}$ & $\underline{69.7_{\pm0.9}}$ & $69.2_{\pm0.6}$ & $\underline{69.7_{\pm0.6}}$ 
    & $76.8_{\pm0.9}$ & $77.4_{\pm0.8}$ & $77.4_{\pm0.5}$ & $77.8_{\pm0.4}$ & $77.8_{\pm0.6}$ \\
    
    ~ +\dropnodef
    & $78.9_{\pm1.3}$ & $78.5_{\pm0.9}$ & $78.0_{\pm0.8}$ & $79.0_{\pm0.9}$ & $79.6_{\pm0.7}$ 
    & $67.0_{\pm0.9}$ & $68.3_{\pm0.7}$ & $67.5_{\pm0.9}$ & $68.6_{\pm1.1}$ & $68.6_{\pm0.8}$ 
    & $75.6_{\pm1.0}$ & $74.6_{\pm0.7}$ & $76.3_{\pm0.8}$ & $76.2_{\pm0.7}$ & $76.3_{\pm0.4}$ \\
    
    ~ +\pairnorm
    & $79.1_{\pm1.3}$ & $79.9_{\pm0.7}$ & $79.2_{\pm1.1}$ & $79.8_{\pm1.0}$ & $79.3_{\pm0.7}$ 
    & $69.3_{\pm1.1}$ & $70.3_{\pm0.7}$ & $69.6_{\pm1.0}$ & $69.5_{\pm0.8}$ & $69.1_{\pm0.5}$ 
    & $74.2_{\pm1.0}$ & $75.4_{\pm0.7}$ & $76.4_{\pm0.6}$ & $77.6_{\pm0.1}$ & $77.9_{\pm0.5}$ \\
    
    ~ +\skipconnection
    & $78.4_{\pm1.2}$ & $78.5_{\pm1.3}$ & $78.4_{\pm0.6}$ & $79.0_{\pm0.7}$ & $79.7_{\pm0.5}$ 
    & $69.0_{\pm0.9}$ & $69.1_{\pm0.7}$ & $69.0_{\pm0.7}$ & $68.9_{\pm0.7}$ & $68.0_{\pm0.9}$ 
    & $77.1_{\pm0.9}$ & $77.2_{\pm0.7}$ & $\underline{77.8_{\pm0.6}}$ & $75.6_{\pm0.1}$ & $75.7_{\pm0.4}$ \\
    
    ~ +\textbf{\skipnodeu}
    & $\textbf{79.7}_{\pm0.6}$ & $\textbf{80.6}_{\pm0.4}$ & $\textbf{80.9}_{\pm0.4}$ & $\textbf{81.2}_{\pm0.5}$ & $\textbf{81.7}_{\pm0.5}$ 
    & $\textbf{70.3}_{\pm0.8}$ & $\textbf{70.7}_{\pm0.6}$ & $\textbf{70.2}_{\pm0.8}$ & $\textbf{70.5}_{\pm0.8}$ & $\textbf{70.0}_{\pm0.9}$ 
    & $\textbf{77.8}_{\pm0.6}$ & $\textbf{77.6}_{\pm0.5}$ & $\textbf{77.9}_{\pm0.2}$ & $\textbf{78.4}_{\pm0.3}$ & $\textbf{78.8}_{\pm0.6}$ \\

    ~ +\textbf{\skipnodeb}
    & $\underline{79.4_{\pm0.4}}$ & $\underline{80.0_{\pm0.3}}$ & $\underline{80.3_{\pm0.2}}$ & $\underline{80.6_{\pm0.4}}$ & $\underline{80.2_{\pm0.3}}$ 
    & $\underline{70.0_{\pm0.6}}$ & $\underline{70.4_{\pm0.5}}$ & $\underline{69.7_{\pm0.6}}$ & $\underline{69.9_{\pm0.5}}$ & $\underline{69.7_{\pm0.6}}$ 
    & $\underline{77.4_{\pm0.4}}$ & $\underline{77.5_{\pm0.4}}$ & $77.6_{\pm0.2}$ & $\underline{77.9_{\pm0.3}}$ & $\underline{78.2_{\pm0.5}}$ \\
    
    \midrule
    GCNII 
     & $82.6_{\pm1.0}$ & $\underline{83.5_{\pm0.9}}$ & $\underline{83.9_{\pm1.3}}$ & $84.1_{\pm0.9}$ & $85.3_{\pm1.3}$ 
     & $68.0_{\pm1.5}$ & $70.8_{\pm0.7}$ & $72.0_{\pm1.5}$ & $72.1_{\pm1.1}$ & $73.3_{\pm1.0}$ 
     & $77.0_{\pm0.8}$ & $77.9_{\pm1.1}$ & $79.8_{\pm0.5}$ & $\underline{80.1_{\pm0.2}}$ & $80.1_{\pm0.2}$ \\

    ~ +\dropedge
     & $82.1_{\pm0.9}$ & $83.2_{\pm1.0}$ & $83.5_{\pm1.6}$ & $83.4_{\pm0.6}$ & $83.0_{\pm0.9}$ 
     & $68.7_{\pm1.4}$ & $70.4_{\pm1.4}$ & $71.2_{\pm1.6}$ & $72.3_{\pm1.2}$ & $72.5_{\pm0.8}$ 
     & $\textbf{77.4}_{\pm0.5}$ & $\textbf{78.3}_{\pm0.5}$ & $79.6_{\pm0.9}$ & $\underline{80.1_{\pm0.3}}$ & $80.0_{\pm0.2}$ \\
    
    ~ +\dropnodef
     & $82.4_{\pm1.0}$ & $81.2_{\pm0.9}$ & $80.2_{\pm1.4}$ & $82.4_{\pm0.7}$ & $82.8_{\pm0.8}$ 
     & $65.8_{\pm1.2}$ & $67.0_{\pm0.9}$ & $68.8_{\pm1.1}$ & $70.7_{\pm0.7}$ & $71.2_{\pm1.0}$ 
     & $75.8_{\pm1.2}$ & $77.5_{\pm1.0}$ & $79.2_{\pm0.7}$ & $79.5_{\pm0.4}$ & $79.6_{\pm0.5}$ \\
    
    ~ +\pairnorm
     & $\underline{82.9_{\pm1.5}}$ & $83.4_{\pm0.8}$ & $83.0_{\pm0.9}$ & $83.5_{\pm1.3}$ & $83.8_{\pm1.3}$ 
     & $69.0_{\pm1.5}$ & $70.8_{\pm0.7}$ & $71.1_{\pm1.5}$ & $72.1_{\pm1.1}$ & $72.3_{\pm1.0}$ 
     & $77.0_{\pm0.8}$ & $77.9_{\pm1.1}$ & $80.0_{\pm0.5}$ & $\underline{80.1_{\pm0.2}}$ & $\underline{80.2_{\pm0.2}}$ \\
    
    ~ +\skipconnection
     &   \multicolumn{1}{c}{\centering -$^{*}$} & \multicolumn{1}{c}{\centering -} & \multicolumn{1}{c}{\centering -} &  \multicolumn{1}{c}{\centering -} &  \multicolumn{1}{c|}{\centering -} &   \multicolumn{1}{c}{\centering -} &   \multicolumn{1}{c}{\centering -} &  \multicolumn{1}{c}{\centering -}  &  \multicolumn{1}{c}{\centering -} &   \multicolumn{1}{c|}{\centering -} &    \multicolumn{1}{c}{\centering -} &    \multicolumn{1}{c}{\centering -} &    \multicolumn{1}{c}{\centering -} &    \multicolumn{1}{c}{\centering -} &  \multicolumn{1}{c}{\centering -} \\
        
    ~ +\textbf{\skipnodeu}
     & $\textbf{83.2}_{\pm1.4}$ & $\textbf{83.9}_{\pm1.0}$ & $\textbf{84.5}_{\pm0.9}$ & $\textbf{84.8}_{\pm0.5}$ & $\textbf{85.7}_{\pm0.4}$ 
     & $\textbf{70.8}_{\pm1.3}$ & $\textbf{71.6}_{\pm1.1}$ & $\textbf{72.9}_{\pm1.1}$ & $\textbf{73.5}_{\pm0.8}$ & $\textbf{74.0}_{\pm1.0}$ 
     & $\underline{77.3_{\pm0.7}}$ & $\textbf{78.3}_{\pm0.7}$ & $\textbf{80.7}_{\pm0.6}$ & $\textbf{80.3}_{\pm0.5}$ & $\textbf{80.6}_{\pm0.2}$ \\

    ~ +\textbf{\skipnodeb}
     & $\underline{82.9_{\pm1.0}}$ & $\underline{83.5_{\pm0.9}}$ & $\underline{83.9_{\pm1.1}}$ & $\underline{84.3_{\pm0.5}}$ & $\underline{85.4_{\pm0.7}}$ 
     & $\underline{70.2_{\pm0.9}}$ & $\underline{71.2_{\pm1.6}}$ & $\underline{72.5_{\pm1.4}}$ & $\underline{72.9_{\pm0.7}}$ & $\underline{73.7_{\pm1.0}}$ 
     & $77.0_{\pm0.9}$ & $\underline{78.0_{\pm0.7}}$ & $\underline{80.4_{\pm0.6}}$ & $\underline{80.1_{\pm0.4}}$ & $\underline{80.2_{\pm0.4}}$ \\
   
    \bottomrule
    
    \addlinespace
    \multicolumn{16}{l}{$^{*}$\footnotesize{GCNII is naturally equipped with \skipconnection.}}
    
    \end{tabular}
    }
  \label{tab:semi} 
\end{table*}

%% file: tables/add_comp.tex
\begin{table}[!htbp]
    \centering
    \caption{\centering{\lwgrevise{Performance Comparison with \dropnodes and \dropmessage.}}}
    \begin{tabular}{ c |c c c c}
    \toprule
     Strategy  & L = 3 & L = 5 & L = 7 & L = 9\\
    \midrule
     - & $81.7_{\pm0.9}$ & $80.6_{\pm0.9}$ & $79.4_{\pm4.3}$ & $77.2_{\pm6.8}$ \\
     
     \dropnodes & $81.1_{\pm1.4}$ & $79.6_{\pm2.3}$ & $46.7_{\pm6.2}$ & $36.2_{\pm7.5}$ \\
     \dropmessage & $82.1_{\pm0.9}$ & $81.1_{\pm1.1}$ & $79.7_{\pm4.5}$ & $77.6_{\pm5.9}$ \\
     \textbf{\skipnodeu} & $\textbf{82.6}_{\pm0.6}$ & $\underline{81.4_{\pm0.5}}$ & $\underline{80.1_{\pm0.9}}$ & $\underline{78.4_{\pm1.1}}$ \\
     \textbf{\skipnodeb} & $\underline{82.4_{\pm0.8}}$ & $\textbf{81.7}_{\pm0.8}$ & $\textbf{81.2}_{\pm0.7}$ & $\textbf{80.2}_{\pm0.9}$ \\
    \bottomrule
    \end{tabular}
    \label{tab:comp_all_strategys} 
\end{table}

%% file: tables/full.tex
\renewcommand{\arraystretch}{}
\begin{table*}[!bpht]
    \centering
    \caption{\centering{\lwgrevise{Generalizability Analysis: Full-supervised Classification Accuracy $\pm$ Standard Deviation (\%).}}}
    \resizebox{1\linewidth}{!}{
    \begin{tabular}{l | c c c c c c c | l |c c c c c c c }
         \toprule 
         GNN Arch. & \multirow{2}{*}{Cora} & \multirow{2}{*}{Cit.} & \multirow{2}{*}{Pub.} & \multirow{2}{*}{Cham.} & \multirow{2}{*}{Corn.} & \multirow{2}{*}{Tex.} & \multirow{2}{*}{Wis.} & GNN Arch. & \multirow{2}{*}{Cora} & \multirow{2}{*}{Cit.} & \multirow{2}{*}{Pub.} & \multirow{2}{*}{Cham.} & \multirow{2}{*}{Corn.} & \multirow{2}{*}{Tex.} & \multirow{2}{*}{Wis.} \\
         
         ~+\texttt{Strategy} & ~ & ~ & ~ & ~ & ~ & ~ & ~ & ~+\texttt{Strategy} & ~ & ~ & ~ & ~ & ~ & ~ & ~ \\
         
        \midrule
         GCN 
         & $86.4_{\pm1.6}$ & $76.5_{\pm1.0}$ & $86.4_{\pm0.5}$ & $43.1_{\pm2.7}$ & $46.7_{\pm5.9}$ & $57.8_{\pm4.7}$ & $50.4_{\pm3.8}$ &
         GCNII 
         & $87.7_{\pm1.5}$ & $77.1_{\pm0.9}$ & $88.9_{\pm0.4}$ & $59.9_{\pm1.6}$ & $72.8_{\pm5.7}$ & $67.8_{\pm3.5}$ & $74.8_{\pm4.1}$ \\
          
         ~+\dropedge
         & $88.5_{\pm1.5}$ & $76.7_{\pm1.2}$ & $86.7_{\pm0.3}$ & $42.6_{\pm2.1}$ & $46.9_{\pm4.8}$ & $58.2_{\pm4.5}$ & $52.2_{\pm5.3}$ &
		  ~+\dropedge
         & $88.0_{\pm1.7}$ & $77.5_{\pm1.2}$ & $87.7_{\pm3.7}$ & $58.4_{\pm2.7}$ & $73.3_{\pm1.8}$ & $65.1_{\pm2.5}$ & $72.6_{\pm4.4}$ \\
         
         ~+\dropnodef 
         & $87.3_{\pm0.9}$ & $75.7_{\pm0.3}$ & $85.9_{\pm0.6}$ & $38.5_{\pm1.0}$ & $44.4_{\pm6.9}$ & $51.9_{\pm7.0}$ & $46.6_{\pm7.2}$ & 
         ~+\dropnodef 
         & $87.4_{\pm1.7}$ & $74.9_{\pm1.6}$ & $85.6_{\pm5.1}$ & $54.3_{\pm4.1}$ & $73.0_{\pm2.4}$ & $63.5_{\pm2.9}$ & $70.8_{\pm2.8}$ \\

         ~+\pairnorm
         & $87.1_{\pm1.2}$ & $76.4_{\pm0.8}$ & $86.3_{\pm0.6}$ & $43.5_{\pm1.6}$ & $46.8_{\pm3.6}$ & $57.6_{\pm2.4}$ & $52.0_{\pm5.4}$ & 
         ~+\pairnorm
         & $\underline{88.6_{\pm1.8}}$ & $77.4_{\pm1.0}$ & $86.9_{\pm1.2}$ & $58.1_{\pm2.2}$ & $\textbf{74.4}_{\pm2.5}$ & $64.9_{\pm2.4}$ & $72.8_{\pm4.5}$ \\
         
         ~+\skipconnection
         & $87.0_{\pm1.1}$ & $76.1_{\pm0.8}$ & $86.8_{\pm0.6}$ & $39.4_{\pm2.2}$ & $46.3_{\pm5.3}$ & $54.6_{\pm2.2}$ & $52.2_{\pm5.5}$ & 
         ~+\skipconnection
         & \multicolumn{1}{c}{\centering -$^{*}$}  & \multicolumn{1}{c}{\centering -}  & \multicolumn{1}{c}{\centering -}  & \multicolumn{1}{c}{\centering -} & \multicolumn{1}{c}{\centering -}  & \multicolumn{1}{c}{\centering -}  & \multicolumn{1}{c}{\centering -}  \\
         
           ~+\textbf{\skipnodeu}
         & $\textbf{89.1}_{\pm1.6}$ & $\textbf{77.2}_{\pm0.9}$ & $\textbf{87.3}_{\pm0.4}$ & $\textbf{51.5}_{\pm1.8}$ & $\underline{50.6_{\pm4.0}}$ & $\textbf{59.1}_{\pm3.9}$ & $\underline{54.2_{\pm4.3}}$ & 
         ~+\textbf{\skipnodeu}
         & $\textbf{88.9}_{\pm1.7}$ & $\textbf{78.1}_{\pm1.1}$ & $\textbf{89.2}_{\pm0.4}$ & $\textbf{60.7}_{\pm2.3}$ & $73.1_{\pm6.2}$ & $\textbf{75.8}_{\pm2.3}$ & $\underline{73.9_{\pm2.6}}$ \\
         
         ~+\textbf{\skipnodeb} 
         & $\underline{88.9_{\pm1.3}}$ & $\underline{76.9_{\pm0.5}}$ & $\underline{87.0_{\pm0.5}}$ & $\underline{47.1_{\pm1.4}}$ & $\textbf{54.2}_{\pm3.8}$ & $\underline{58.9_{\pm2.8}}$ & $\textbf{56.2}_{\pm4.4}$
		 & ~+\textbf{\skipnodeb} 
		 & $\underline{88.6_{\pm1.2}}$ & $\underline{77.8_{\pm1.0}}$ & $\underline{89.0_{\pm0.7}}$ & $\underline{60.1_{\pm1.2}}$ & $\underline{73.8_{\pm5.7}}$ & $\underline{72.1_{\pm3.5}}$ & $\textbf{74.1}_{\pm3.7}$ \\

         \midrule
         
         GAT
         & $86.1_{\pm1.5}$ & $75.5_{\pm1.0}$ & $86.2_{\pm0.4}$ & $45.9_{\pm2.3}$ & $52.5_{\pm6.6}$ & $\underline{58.9_{\pm3.8}}$ & $53.2_{\pm4.7}$ & 
         GRAND 
         & $87.9_{\pm1.5}$ & $76.7_{\pm1.2}$ & $\underline{84.3_{\pm0.4}}$ & $47.3_{\pm1.7}$ & $58.3_{\pm5.4}$ & $58.1_{\pm3.7}$ & $56.6_{\pm4.3}$ \\

         ~+\dropedge 
         & $86.7_{\pm1.6}$ & $76.2_{\pm0.8}$ & $85.9_{\pm0.5}$ & $45.2_{\pm1.8}$ & $\underline{53.6_{\pm4.3}}$ & $58.1_{\pm4.6}$ & $53.0_{\pm8.2}$ & 
         ~+\dropedge 
         & $87.4_{\pm1.9}$ & $76.3_{\pm1.1}$ & $83.5_{\pm0.6}$ & $47.2_{\pm1.2}$ & $57.1_{\pm3.6}$ & $57.6_{\pm2.5}$ & $\underline{56.9_{\pm2.4}}$  \\
         
         ~+\dropnodef 
         & $85.4_{\pm1.7}$ & $74.8_{\pm1.2}$ & $85.1_{\pm0.5}$ & $44.1_{\pm3.2}$ & $47.5_{\pm7.2}$ & $54.3_{\pm5.5}$ & $46.2_{\pm7.9}$ & 
         ~+\dropnodef 
         & \multicolumn{1}{c}{\centering -$^{*}$} & \multicolumn{1}{c}{\centering -} & \multicolumn{1}{c}{\centering -} & \multicolumn{1}{c}{\centering -} & \multicolumn{1}{c}{\centering -} & \multicolumn{1}{c}{\centering -} & \multicolumn{1}{c}{\centering -} \\

         ~+\pairnorm
         & $87.4_{\pm1.3}$ & $75.7_{\pm0.9}$ & $85.9_{\pm0.5}$ & $45.7_{\pm1.6}$ & $52.2_{\pm6.1}$ & $56.2_{\pm3.6}$ & $52.4_{\pm7.1}$ & 
         ~+\pairnorm 
         & $87.4_{\pm1.8}$ & $76.0_{\pm0.8}$ & $84.1_{\pm0.5}$ & $46.9_{\pm2.1}$ & $57.4_{\pm2.9}$ & $58.7_{\pm3.2}$ & $56.8_{\pm1.4}$ \\
         
          ~+\skipconnection
         & $87.1_{\pm1.0}$ & $76.0_{\pm1.1}$ & $85.3_{\pm0.2}$ & $45.1_{\pm1.4}$ & $50.3_{\pm6.4}$ & $56.2_{\pm3.6}$ & $52.4_{\pm7.0}$ &  
         ~+\skipconnection
         & $87.5_{\pm1.7}$ & $76.0_{\pm0.9}$ & $83.4_{\pm0.4}$ & $46.1_{\pm1.1}$ & $57.9_{\pm4.2}$ & $57.7_{\pm3.2}$ & $56.2_{\pm2.6}$ \\

         ~+\textbf{\skipnodeu}
         & $\textbf{88.7}_{\pm1.2}$ & $\underline{76.9_{\pm1.0}}$ & $\textbf{88.1}_{\pm0.5}$ & $\textbf{51.9}_{\pm1.6}$ & $\textbf{54.6}_{\pm4.6}$ & $\textbf{59.7}_{\pm3.3}$ & $\textbf{54.4}_{\pm4.4}$	& 
         ~+\textbf{\skipnodeu}
         & $\underline{88.1_{\pm1.7}}$ & $\textbf{77.1}_{\pm1.1}$ & $\textbf{84.4}_{\pm0.4}$ & $\textbf{47.8}_{\pm1.6}$ & $\textbf{58.9}_{\pm3.8}$ & $\underline{59.5_{\pm2.8}}$ & $56.8_{\pm3.5}$ \\

         ~+\textbf{\skipnodeb}
         & $\underline{88.0_{\pm1.5}}$ & $\textbf{77.1}_{\pm0.8}$ & $\underline{87.2_{\pm0.6}}$ & $\underline{48.2_{\pm2.3}}$ & $53.1_{\pm5.2}$ & $\textbf{59.7}_{\pm3.3}$ & $\underline{54.2_{\pm4.8}}$ & 
         ~+\textbf{\skipnodeb}
         & $\textbf{88.6}_{\pm1.2}$ & $\underline{76.8_{\pm1.0}}$ & $\underline{84.3_{\pm0.1}}$ & $\underline{47.6_{\pm1.4}}$ & $\underline{58.6_{\pm3.2}}$ & $\textbf{60.1}_{\pm2.2}$ & $\textbf{57.3}_{\pm3.5}$ \\
         
         \midrule
         
         JKNet 
         & $88.2_{\pm1.5}$ & $76.9_{\pm1.0}$ & $88.5_{\pm0.5}$ & $46.7_{\pm1.6}$ & $70.2_{\pm3.8}$ & $68.8_{\pm3.9}$ & $68.8_{\pm5.7}$ & 
         GPRGNN 
         & $88.2_{\pm1.1}$ & $77.0_{\pm0.8}$ & $88.1_{\pm0.3}$ & $53.5_{\pm1.9}$ & $66.4_{\pm4.9}$ & $74.3_{\pm4.4}$ & $72.8_{\pm3.2}$ \\

         ~+\dropedge 
         & $88.0_{\pm1.3}$ & $76.0_{\pm1.0}$ & $88.2_{\pm0.5}$ & $45.2_{\pm1.3}$ & $71.2_{\pm6.1}$ & $67.2_{\pm4.4}$ & $\underline{69.4_{\pm5.4}}$ & 
         ~+\dropedge 
         & $\underline{88.3_{\pm1.6}}$ & $77.8_{\pm1.2}$ & $88.5_{\pm0.2}$ & $52.5_{\pm1.6}$ & $65.6_{\pm4.5}$ & $73.0_{\pm4.4}$ & $71.2_{\pm4.7}$ \\
         
         ~+\dropnodef 
         & $87.3_{\pm1.4}$ & $75.1_{\pm0.7}$ & $86.9_{\pm0.5}$ & $44.4_{\pm2.2}$ & $66.1_{\pm4.8}$ & $65.5_{\pm3.8}$ & $66.8_{\pm5.0}$ & 
         ~+\dropnodef 
         & $87.7_{\pm1.8}$ & $75.8_{\pm1.2}$ & $86.5_{\pm0.2}$ & $50.5_{\pm1.4}$ & $60.2_{\pm8.3}$ & $70.2_{\pm4.1}$ & $68.6_{\pm8.9}$ \\

         ~+\pairnorm 
         & $87.5_{\pm1.7}$ & $75.8_{\pm1.0}$ & $87.6_{\pm0.4}$ & $46.1_{\pm1.0}$ & $69.7_{\pm4.3}$ & $68.4_{\pm4.4}$ & $68.4_{\pm5.3}$ & 
         ~+\pairnorm 
         & $\underline{88.3_{\pm1.8}}$ & $\underline{78.2_{\pm1.0}}$ & $88.3_{\pm0.2}$ & $53.4_{\pm1.2}$ & $66.8_{\pm4.2}$ & $73.7_{\pm5.8}$ & $70.4_{\pm4.4}$ \\
         
         ~+\skipconnection
         & $87.6_{\pm1.7}$ & $75.2_{\pm1.0}$ & $87.2_{\pm0.5}$ & $46.1_{\pm1.2}$ & $68.7_{\pm4.3}$ & $68.2_{\pm3.8}$ & $67.9_{\pm5.3}$ & 
         ~+\skipconnection
         & $88.1_{\pm1.1}$ & $77.0_{\pm1.0}$ & $88.0_{\pm0.5}$ & $53.2_{\pm1.8}$ & $65.8_{\pm6.2}$ & $72.7_{\pm4.3}$ & $70.4_{\pm3.9}$\\

         ~+\textbf{\skipnodeu}
         & $\underline{89.2_{\pm1.6}}$ & $\textbf{77.4}_{\pm1.0}$ & $\textbf{89.1}_{\pm0.4}$ & $\textbf{47.7}_{\pm1.8}$ & $\textbf{74.8}_{\pm5.9}$ & $\textbf{71.4}_{\pm3.1}$ & $\textbf{70.2}_{\pm5.3}$
         & ~+\textbf{\skipnodeu}
         & $\textbf{89.4}_{\pm1.4}$ & $\textbf{78.6}_{\pm1.3}$ & $\textbf{89.7}_{\pm0.3}$ & $\textbf{55.1}_{\pm1.0}$ & $\textbf{68.9}_{\pm5.8}$ & $\textbf{76.2}_{\pm4.8}$ & $\textbf{74.4}_{\pm5.3}$ \\

         ~+\textbf{\skipnodeb}
         & $\textbf{89.6}_{\pm1.8}$ & $\underline{77.2_{\pm1.1}}$ & $\underline{88.7_{\pm0.5}}$ & $\underline{46.9_{\pm1.4}}$ & $\underline{73.6_{\pm4.3}}$ & $\underline{70.2_{\pm4.7}}$ & $69.3_{\pm4.7}$
        &  ~+\textbf{\skipnodeb}
         & $\textbf{89.4}_{\pm1.2}$ & $78.1_{\pm0.8}$ & $\underline{89.3_{\pm0.5}}$ & $\underline{54.7_{\pm1.3}}$ & $\underline{68.1_{\pm4.1}}$ & $\underline{75.1_{\pm4.2}}$ & $\underline{73.9_{\pm2.1}}$ \\
         
         \midrule
         
         InceptGCN 
         & $86.9_{\pm1.4}$ & $76.2_{\pm0.8}$ & $87.8_{\pm0.4}$ & $47.6_{\pm1.8}$ & $71.6_{\pm6.1}$ & $69.1_{\pm2.5}$ & $67.6_{\pm6.1}$ & 
         APPNP 
         & $87.8_{\pm1.6}$ & $78.0_{\pm1.0}$ & $86.2_{\pm0.3}$ & $52.9_{\pm2.0}$ & $70.1_{\pm2.8}$ & $68.9_{\pm3.8}$ & $69.2_{\pm4.4}$ \\

         ~+\dropedge 
         & $86.8_{\pm1.4}$ & $76.2_{\pm1.2}$ & $\textbf{88.2}_{\pm0.2}$ & $46.8_{\pm1.7}$ & $70.1_{\pm6.9}$ & $69.6_{\pm3.1}$ & $68.2_{\pm5.9}$&  
         ~+\dropedge 
         &  $87.4_{\pm1.3}$ & $78.4_{\pm1.2}$ & $86.7_{\pm0.4}$ & $51.1_{\pm2.4}$ & $71.9_{\pm3.5}$ & $68.4_{\pm4.4}$ & $69.0_{\pm4.7}$ \\
         
         ~+\dropnodef 
         & $86.7_{\pm1.3}$ & $75.9_{\pm0.8}$ & $86.2_{\pm0.3}$ & $45.4_{\pm2.2}$ & $68.4_{\pm5.3}$ & $68.9_{\pm2.4}$ & $66.6_{\pm5.7}$ & 
         ~+\dropnodef 
         & $85.4_{\pm1.5}$ & $76.1_{\pm1.6}$ & $85.5_{\pm0.3}$ & $49.3_{\pm2.9}$ & $68.1_{\pm4.8}$ & $66.2_{\pm6.1}$ & $67.2_{\pm5.5}$ \\

         ~+\pairnorm 
         & $\underline{87.4_{\pm1.4}}$ & $76.3_{\pm0.9}$ & $87.9_{\pm0.2}$ & $\underline{47.8_{\pm1.2}}$ & $71.3_{\pm5.0}$ & $69.4_{\pm3.6}$ & $67.5_{\pm2.6}$ & 
         ~+\pairnorm 
         & $87.7_{\pm1.2}$ & $78.6_{\pm1.1}$ & $86.9_{\pm0.4}$ & $51.8_{\pm1.8}$ & $70.5_{\pm3.1}$ & $68.6_{\pm4.0}$ & $69.2_{\pm4.2}$ \\
         
         ~+\skipconnection
         & $86.4_{\pm1.4}$ & $76.1_{\pm0.9}$ & $87.6_{\pm0.2}$ & $47.7_{\pm1.3}$ & $70.9_{\pm4.7}$ & $68.7_{\pm3.6}$ & $67.0_{\pm3.8}$ &   
         ~+\skipconnection
         & \multicolumn{1}{c}{\centering -$^{*}$} & \multicolumn{1}{c}{\centering -} & \multicolumn{1}{c}{\centering -} & ~ & \multicolumn{1}{c}{\centering -} & \multicolumn{1}{c}{\centering -} & \multicolumn{1}{c}{\centering -} \\

         ~+\textbf{\skipnodeu}
         & $\textbf{87.7}_{\pm1.6}$ & $\underline{77.1_{\pm0.9}}$ & $\underline{88.1_{\pm0.2}}$ & $\textbf{48.2}_{\pm1.4}$ & $\textbf{74.2}_{\pm5.6}$ & $\underline{70.2_{\pm3.6}}$ & $\underline{68.4_{\pm4.9}}$ &  
         ~+\textbf{\skipnodeu}
         & $\textbf{88.4}_{\pm1.7}$ & $\underline{79.3_{\pm1.2}}$ & $\textbf{87.7}_{\pm0.6}$ & $\textbf{53.4}_{\pm1.6}$ & $\textbf{74.4}_{\pm1.6}$ & $\textbf{69.8}_{\pm1.8}$ & $\textbf{70.7}_{\pm2.8}$ \\

         ~+\textbf{\skipnodeb}
         & $87.2_{\pm1.5}$ & $\textbf{77.2}_{\pm0.8}$ & $87.9_{\pm0.4}$ & $47.7_{\pm1.9}$ & $\underline{72.1_{\pm4.9}}$ & $\textbf{72.1}_{\pm2.2}$ & $\textbf{69.2}_{\pm3.4}$
         & ~+\textbf{\skipnodeb}
         & $\underline{88.1_{\pm1.6}}$ & $\textbf{80.1}_{\pm0.7}$ & $\underline{87.1_{\pm0.4}}$ & $\underline{53.1_{\pm1.8}}$ & $\underline{72.1_{\pm0.7}}$ & $\underline{69.1_{\pm1.2}}$ & $\underline{70.2_{\pm2.1}}$ \\

         \bottomrule
         
         \addlinespace
    \multicolumn{16}{l}{$^{*}$\footnotesize{Both GCNII and APPNP naturally incorporate \skipconnection, while GRAND incorporates \dropnodef as a part of its design.}}
         
    \end{tabular}
    }
    \label{table:full} 
\end{table*}

%% file: chapters/appendix.tex
\appendices
\section{Proof of Lemma 1}
\label{supp_sec:lemma1}

\begin{assumption}
\label{supp:asp1}
$U$ has an orthonormal basis $(e_{m})_{m \in [M]}$ that consists of non-negative vectors.
\end{assumption} 

\begin{assumption} 
\label{supp:asp2}
$U$ is invariant under $\tilde{A}$, i.e., if $u \in U$, then $\tilde{A}u \in U$.
\end{assumption} 

\begin{lemma}
\label{supp:lem1}
For any $Z$, $Y \in \mathbb{R}^{N \times d}$, we have $\DMtwo{Z + Y} \leq \DMtwo{Z} + \DMtwo{Y} + 2\DMone{Z}\DMone{Y}$ and $\DMtwo{Z - Y} \geq \DMtwo{Z} + \DMtwo{Y} - 2\DMone{Z}\DMone{Y}$.
\end{lemma}

\begin{proof}[Proof of Lemma~\ref{lem1}]
\label{supp:proof_of_lemma2}
Let $(e_{m})_{m\in[N]}$ denote the orthonormal basis of $\mathbb R^{N}$, where $e_{m}$'s are the eigenvectors of $\tilde{A}$. Accordingly, we have $Z = \sum_{m=1}^{N} e_{m} \otimes w^{Z}_{m}$ and $Y =\sum_{m=1}^{N} e_{m} \otimes w^{Y}_{m}$ for some $w^{Z}_{m}, w^{Y}_{m} \in \mathbb{R}^{d}$. Then, we have
\begin{small}
\begin{align*}
    \DMtwo{Z} 
        &= inf\{\lVert Z - H \rVert^{2}_{F} | H \in \mathcal{M}\} \\
        &= inf\{\lVert \sum_{m=1}^{N} e_{m} \otimes w^{Z}_{m} - \sum_{m=1}^{M} e_{m} \otimes \tilde{w}_{m} \rVert^{2}_{F} | \tilde{w}_{m} \in \mathbb{R}^{d}\} \\
        &= inf\{\lVert \sum_{m=1}^{M} e_{m} \otimes \left(w^{Z}_{m} - \tilde{w}_{m}\right) \\
        &+ \sum_{m=M+1}^{N} e_{m} \otimes w^{Z}_{m} \rVert^{2}_{F} | \tilde{w}_{m} \in \mathbb{R}^{d}\} \\
        &= \lVert \sum_{m=M+1}^{N} e_{m} \otimes w^{Z}_{m} \rVert^{2}_{F} \\
        &= \sum_{m=M+1}^{N} \lVert e_{m} \otimes w^{Z}_{m} \rVert^{2}_{F} \quad (\because e_m \perp e_n, \forall m \neq n) \\
        &= \sum_{m=M+1}^{N} \lVert w_{m}^{Z}\rVert^{2}.
\end{align*}
\end{small}
Accordingly, we have $\DMtwo{Y} = \sum_{m=M+1}^{N} \lVert w_{m}^{Y}\rVert^{2}$. Then, with the help of Cauchy–Schwarz inequality, we can derive
\begin{small}
\begin{align*}
    \DMtwo{Z+Y} 
        &= \sum_{m=M+1}^{N} \lVert w_{m}^{Z} + w_{m}^{Y} \rVert^{2} \\
        &= \sum_{m=M+1}^{N} \left(\lVert w_{m}^{Z}\rVert^{2} + \lVert w_{m}^{Y} \rVert^{2} + 2 (w_{m}^{Z})^T (w_{m}^{Y}) \right)\\
        &\leq \sum_{m=M+1}^{N} \lVert w_{m}^{Z}\rVert^{2} + \sum_{m=M+1}^{N} \lVert w_{m}^{Y} \rVert^{2} \\
        &\quad + 2 \sum_{m=M+1}^{N} \lVert w_{m}^{Z}\rVert \lVert w_{m}^{Y} \rVert\\
        &\leq \DMtwo{Z} + \DMtwo{Y}  \\
        &+ 2\sqrt{\sum_{m=M+1}^{N} \lVert w_{m}^{Z}\rVert^{2} \sum_{m=M+1}^{N} \lVert w_{m}^{Y}\rVert^{2}} \\
        &= \DMtwo{Z} + \DMtwo{Y} + 2\DMone{Z}\DMone{Y},
\end{align*}
\end{small}
and
\begin{footnotesize}
\begin{align*}
    \DMtwo{Z-Y}
        &= \sum_{m=M+1}^{N} \lVert w_{m}^{Z} - w_{m}^{Y} \rVert^{2} \\
        &\geq \sum_{m=M+1}^{N} \lVert w_{m}^{Z}\rVert^{2} + \sum_{m=M+1}^{N} \lVert w_{m}^{Y} \rVert^{2} \\
        &\quad- 2 \sum_{m=M+1}^{N} \lVert w_{m}^{Z}\rVert \lVert w_{m}^{Y} \rVert\\
        %&\geq \DMtwo{Z} + \DMtwo{Y} - 2\sqrt{\DMtwo{Z}\DMtwo{Y}} \\
        &\geq \DMtwo{Z} + \DMtwo{Y} - 2\DMone{Z}\DMone{Y}.
\end{align*}
\end{footnotesize}

\end{proof}

\section{Proof of Theorem 1}
\label{supp_sec:theorem1}

\begin{theorem}
\label{supp:thm_gradient_vanishing}
Let $\mathcal{L}$ denote the cross-entropy loss function, $\mathcal{V}_{train}$ be the training set that contains $B$ nodes and $Z \in \mathbb{R}^{N \times C}$ be the output of the classification layer, where $C$ is the number of class. Assuming there are same samples in each category ($\frac{B}{C}$ samples per class), the gradient at the classification layer $\sum_{i \in \mathcal{V}_{train}} \sum_{j=1}^{C} \frac{\partial \mathcal{L}}{\partial Z_{ij}}$ gets close to 0 when the model's output converges to 0, the trivial fixed point introduced by the Proposition 3 from~\cite{oono2019graph}.
\end{theorem}

\begin{proof}
Let $\tilde{Y} = softmax(Z)$ denote the prediction score, $Y \in \mathbb R^{N \times C}$ denote the ground truth indicator, and $\mathcal{L}$ denote the cross-entropy loss function. Then, we have:
\begin{small}
    \begin{align*}
    \mathcal{L} &= -\frac{1}{B}\sum_{v_i \in \mathcal{V}_{train}} \sum_{j=1}^{C} Y_{ij}\log\tilde{Y}_{ij}\\
      &= -\frac{1}{B}\sum_{v_i \in \mathcal{V}_{train}} \sum_{j=1}^{C} Y_{ij}\log\frac{exp(Z_{ij})}{\sum_{k=1}^{C} exp(Z_{ik})},
\end{align*}
\end{small}
where $\mathcal{V}_{train}$ indicates the training set. 

For a node $v_i \in \mathcal{V}_{train}$ that belongs to the $c$-th class, we have $Y_{ic} = 1$ and $Y_{ij} = 0, \forall j \ne c$. Then, the gradient of $Z_{ic}$ is calculated as:
\begin{small}
\begin{align*}
    \frac{\partial \mathcal{L}}{\partial Z_{ic}} &= 
    -\frac{1}{B} \frac{1}{\tilde{Y}_{ic}}\left(\frac{exp(Z_{ic})\sum_{k=1}^{C} exp(Z_{ik}) - exp(Z_{ic})^{2} }{[\sum_{k=1}^{C} exp(Z_{ik})]^{2}}\right)\\
    &= -\frac{1}{B}\frac{1}{\tilde{Y}_{ic}}\left(\tilde{Y}_{ic} - \tilde{Y}_{ic}^{2}\right)\\
    &= \frac{\tilde{Y}_{ic} - 1}{B},
\end{align*}
\end{small}
and the gradient of $Z_{ij}, \forall j \ne c$ is calculated as:
\begin{small}
    \begin{align*}
    \frac{\partial \mathcal{L}}{\partial Z_{ij}} &= 
    \frac{1}{B} \frac{1}{\tilde{Y}_{ic}}\frac{exp(Z_{ic})exp(Z_{ij})}{\left(\sum_{k=1}^{C} exp(Z_{ik})\right)^{2}}\\
    &= \frac{1}{B}\frac{1}{\tilde{Y}_{ic}}\tilde{Y}_{ic}\tilde{Y}_{ij}\\
    &= \frac{\tilde{Y}_{ij}}{B}.
\end{align*}
\end{small}

Under the assumption in the Proposition 3 from~\cite{oono2019graph}, $Z$ converges to 0 as $l \to \infty$. Then, we have $\tilde{Y}_{ij} \approx \frac{1}{C}$. Accordingly, we can further derive the gradient as follows:

\begin{footnotesize}
\begin{align*}
    \sum_{i \in \mathcal{V}_{train}}\frac{\partial \mathcal{L}}{\partial Z_{ij}} &= \sum_{i \in \mathcal{V}_{train}} \left(\frac{Y_{ij}\left(\tilde{Y}_{ij} - 1\right)}{B} + \frac{\left(1 - Y_{ij}\right)\tilde{Y}_{ij}}{B}\right)\\
    &\approx \frac{1}{B} \sum_{i \in \mathcal{V}_{train}} \left(Y_{ij}\left(\frac{1}{C} - 1\right) + \left(1 - Y_{ij}\right)\frac{1}{C}\right)\\
    &= \frac{1}{B} \sum_{i \in \mathcal{V}_{train}}\left(\frac{Y_{ij}}{C} - Y_{ij} + \frac{1}{C} - \frac{Y_{ij}}{C}\right)\\
    &= \frac{1}{B} \sum_{i \in \mathcal{V}_{train}}\left(\frac{1}{C} - Y_{ij}\right)\\
    &= \frac{1}{B} \left(\frac{B}{C} - \sum_{i \in \mathcal{V}_{train}}Y_{ij}\right)\\
    &= \frac{1}{B} \left(\frac{B}{C} - \frac{B}{C}\right)\\
    &= 0.
\end{align*}
\end{footnotesize}
As a result, the gradient at the classification layer $\sum_{i \in \mathcal{V}_{train}} \sum_{j=1}^{C} \frac{\partial \mathcal{L}}{\partial Z_{ij}} = 0$. 
\end{proof}

\section{Proof of Theorem 2}
\label{supp_sec:theorem2}

\begin{theorem}[Higher Upper Bound]
\label{supp:thm1}
Under Assumptions~\ref{asp1} and \ref{asp2}, for any $\rho \in (0,1)$, we have $\DMone{\mathbb{E}[X_{2}]} \leq \left(s\lambda + \rho(1 - s\lambda)\right)\DMone{X}$ compared to $\DMone{X_{1}} \leq s\lambda\DMone{X}$ (cf. Theorem 1 of \cite{oono2019graph}). In particular, when $s\lambda < 1$, \skipnode can achieve a higher upper bound, i.e., $s\lambda + \rho(1 - s\lambda) > s\lambda$.
\end{theorem}

\begin{proof}[Proof of Theorem \ref{thm1}]
\label{supp:proof1}
Since $P_{ii}$ is a Bernoulli random variable such that $P_{ii} \sim Bernoulli(\rho)$, we have $\mathbb{E}[X_{2}] = (1-\rho)X_{1} + \rho X$. Following the decomposition scheme used in Lemma~\ref{lem1}, we let $X = \sum_{m=1}^{N} e_{m} \otimes w_{m}$ and $X_{1} = \sum_{m=1}^{N} e_{m} \otimes w^{\prime}_{m}$ for some $w_{m}, w^{\prime}_{m} \in \mathbb R^{d}$. Then, we have
\begin{small}
\begin{align*}
    \DMtwo{\mathbb{E}[X_{2}]} 
        &= \DMtwo{(1-\rho)X_{1} + \rho X} \\
        &\leq \DMtwo{(1-\rho)X_{1}} + \DMtwo{\rho X} \\
        &\quad+ 2\DMone{(1-\rho)X_{1}}\DMone{\rho X} \quad (\because Lemma~\ref{lem1})\\
        &= (1-\rho)^{2}\DMtwo{X_{1}} + \rho^{2}\DMtwo{X} \\
        &\quad+ 2\rho(1-\rho)\DMone{X_{1}}\DMone{X}\\
        &\leq (1-\rho)^{2}s^{2}\lambda^{2}\DMtwo{X} + \rho^{2}\DMtwo{X}\\
        &\quad+ 2\rho(1-\rho)s\lambda \DMtwo{X} \quad (\because \DMone{X_{1}} \leq s\lambda\DMone{X})\\
        &= \left((1-\rho)s\lambda + \rho\right)^{2}\DMtwo{X}.
\end{align*}
\end{small}

% Here, we consider two cases: (1)$(1-\rho)s\lambda + \rho < 0$ and (2) $(1-\rho)s\lambda + \rho \geq 0$ as follows.

% \textbullet
% \textbf{CASE 1:} $(1-\rho)s\lambda + \rho < 0$

% However, this case cannot be achieved since $\rho \in (0,1)$ and $s\lambda > 0$.

% \textbullet
% \textbf{CASE 2:} $(1-\rho)s\lambda + \rho \geq 0$

Since $\rho \in (0,1)$ and $s\lambda > 0$, we have
\begin{small}
\begin{align*}
    \DMone{\mathbb{E}[X_{2}]}  &\leq \left((1-\rho)s\lambda + \rho\right)\DMone{X} \\
    &= \left(s\lambda + \rho(1 - s\lambda)\right)\DMone{X}.
\end{align*}
\end{small}
Especially, when $s\lambda < 1$, we have $s\lambda + \rho (1- s\lambda) > s\lambda$ indicating that \skipnode can achieve a higher upper bound (in expectation) than the vanilla GCN.
\end{proof}

\section{Proof of Theorem 3}
\label{supp_sec:theorem3}

\begin{theorem}[Longer Distance from $\mathcal{M}$]
\label{supp:thm2}
Under Assumptions~\ref{asp1} and \ref{asp2}, when $\rho(\frac{1}{s\lambda} + 1) - 1 > 0$, \skipnode can guarantee a lower bound as $\DMone{\mathbb{E}[X_{2}]} \geq \left( \rho(\frac{1}{s\lambda} + 1) - 1\right)\DMone{X_{1}}$. In particular, when $\rho(\frac{1}{s\lambda} + 1) > 2$, the output of \skipnode is farther away from $\mathcal{M}$, i.e., $\frac{d_{\mathcal{M}}(\mathbb{E}[X_{2}])}{d_{\mathcal{M}}(X_{1})} > 1$.
\end{theorem}

\begin{proof}[Proof of Theorem~\ref{thm2}]
\label{supp:proof2}
Similar to Theorem~\ref{thm1}, we have
\begin{small}
    \begin{align*}
    \DMtwo{\mathbb{E}[X_{2}]} &= d^{2}\left((1-\rho)X_{1} + \rho X\right) \\
    &= \DMtwo{\rho X - (\rho - 1)X_{1}} \\
    &\geq \DMtwo{\rho X} + \DMtwo{(\rho-1)X_{1}} \\
    &\quad- 2\DMone{\rho X} \DMone{(\rho-1)X_{1}} \quad (\because Lemma~\ref{lem1})\\
    &= \rho^{2}\DMtwo{X} + (\rho - 1)^{2}\DMtwo{X_{1}} \\
    &\quad- 2|\rho (\rho - 1)| \DMone{X}\DMone{X_{1}} \\
    &= \left(\rho\DMone{X} - (1 - \rho) \DMone{X_{1}}\right)^{2}.
    \end{align*}
\end{small}

Since the operand of the square could be negative, we consider the following two cases (the ``='' case is not considered since it results in a trivial lower bound, i.e., 0).

\textbullet\
\textbf{CASE 1:} $\rho\DMone{X} < (1 - \rho) \DMone{X_{1}}$

It implies $\frac{1}{s\lambda} \leq \frac{\DMone{X}}{\DMone{X_{1}}} < \frac{1-\rho}{\rho}$. Therefore, $\frac{1-\rho}{\rho} > \frac{1}{s\lambda}$ is the necessary condition for this case. The lower bound becomes $\DMone{\mathbb{E}[X_{2}]} \geq (1 - \rho)\DMone{X_{1}} - \rho \DMone{X} = \DMone{X_{1}} -\rho (\DMone{X_{1}} + \DMone{X})$, which is definitely less than $\DMone{X_{1}}$. Hence, this case cannot guarantee $\DMone{\mathbb{E}[X_{2}]} \geq \DMone{X_{1}}$, no matter how we set $\rho$.

%However, in this case, we cannot guarantee that it is a reasonable lower bound since
% \begin{small}
%     \begin{align*}
%         (1 - \rho)\DMone{X_{1}} - \rho \DMone{X} 
%             &\leq (1 - \rho)\DMone{X_{1}} - \frac{\rho}{s\lambda} \DMone{X_{1}} \\
%             &= \left(1 - \rho(\frac{1}{s\lambda} + 1) \right)\DMone{X_{1}},\\
%     \end{align*}
% \end{small}
% where $1 - \rho(\frac{1}{s\lambda} + 1) < 1$.

\textbullet\
\textbf{CASE 2:} $\rho\DMone{X} > (1 - \rho) \DMone{X_{1}}$

It indicates $\frac{\DMone{X}}{\DMone{X_{1}}} > \frac{1 - \rho}{\rho}$. Since $\frac{\DMone{X}}{\DMone{X_{1}}} \geq \frac{1}{s\lambda}$, $\frac{1}{s\lambda} > \frac{1 - \rho}{\rho}$ is a sufficient condition for this case. This condition is also required to assure the following lower bound nontrivial (i.e., with simple transformation, we have $\rho(\frac{1}{s\lambda} + 1) - 1 > 0$):
\begin{small}
    \begin{align*}
    \DMone{\mathbb{E}[X_{2}]} 
        &\geq \rho\DMone{X} - (1 - \rho) \DMone{X_{1}} \\
        &\geq \left( \rho(\frac{1}{s\lambda} + 1) - 1\right)\DMone{X_{1}}.
    \end{align*}
\end{small}
In particular, when $\rho(\frac{1}{s\lambda} + 1) > 2$, we can guarantee $\DMone{\mathbb{E}[X_{2}]} \geq \DMone{X_{1}}$ because:
\begin{small}
    \begin{align*}
    \frac{\DMone{\mathbb{E}[X_{2}]}}{\DMone{X_{1}}} \geq \rho(\frac{1}{s\lambda} + 1) - 1 > 1.
\end{align*}
\end{small}
\end{proof}

\section{\lwgrevise{Empirical Guidance of \skipnode}}
\label{supp_sec:guidance}
\lwgrevise{
Here, we provide empirical guidance of our \skipnode by taking two key aspects into consideration:}
\begin{itemize}
	\item \lwgrevise{\textbf{Model:}  If the models are specifically designed to address over-smoothing, such as GCNII~\cite{gcnii} and JKNet~\cite{jknet}, or models adopting a relatively shallow architecture (4\textasciitilde{}8 layers), we recommend utilizing SkipNode-U. However, if the models are not explicitly tailored for over-smoothing, like GCN and GAT, especially in deep architectures, we recommend employing SkipNode-B. The biased sampling in SkipNode-B takes into account the degree of nodes, as high-degree nodes are more susceptible to over-smoothing, as elaborated in~\cite{gcnii}. Since this category of models does not give any treatment for over-smoothing, it is important to pay more attention to high-degree nodes.}
	
	\item \lwgrevise{\textbf{Dataset:} It is worth noting that denser graphs often present tougher challenges when deepening GCNs, as the increased number of message-passing operations between connected nodes can easily lead to over-smoothed features. In such cases, we recommend adopting SkipNode-B, which considers degree-biased sampling. }
\end{itemize}

%% file: chapters/biography.tex
\vspace{-1cm}

\begin{IEEEbiography}[{\includegraphics[width=0.9in, height=1.25in,clip,keepaspectratio]{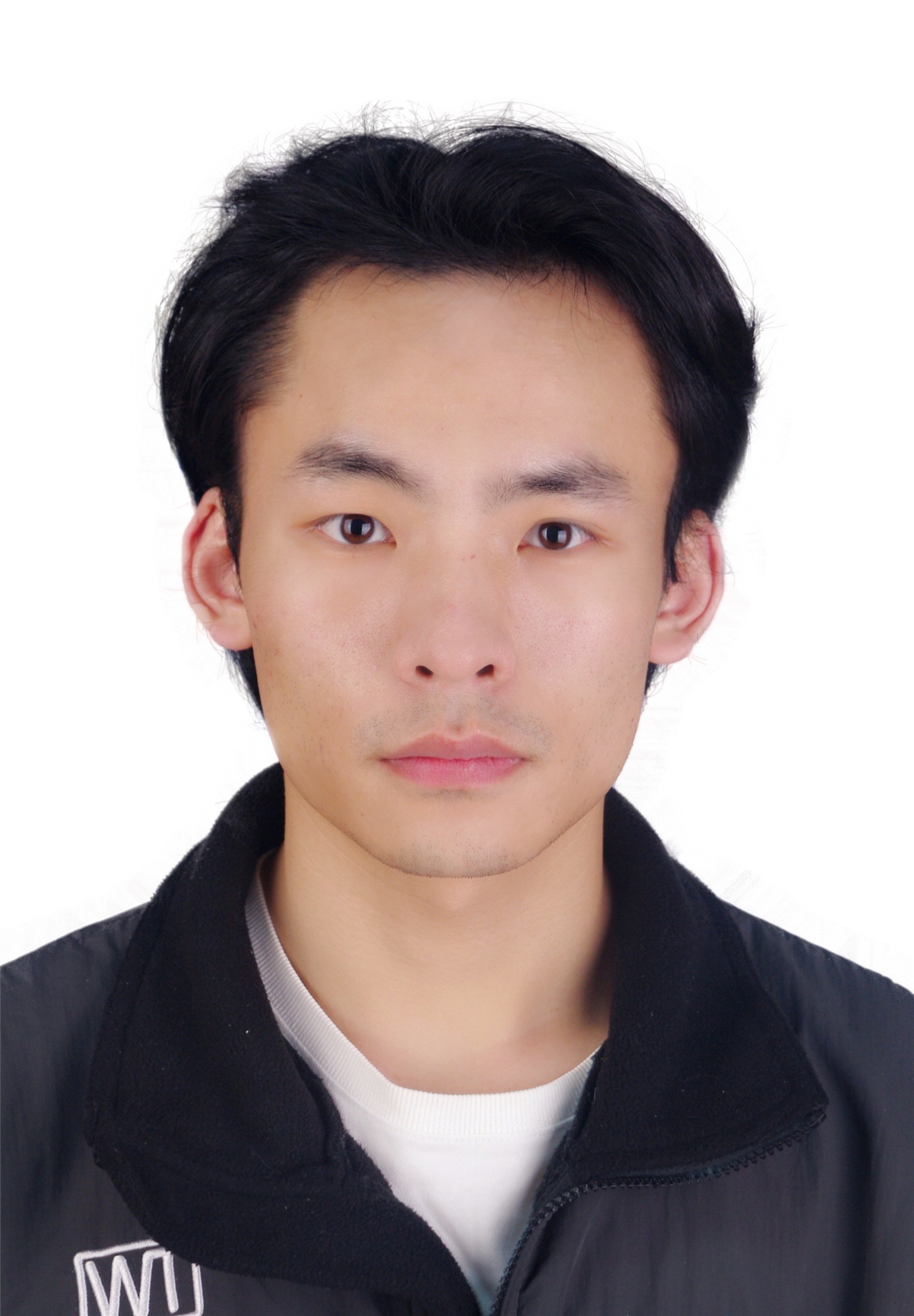}}]{Weigang Lu} received the B.S. degree in Internet of Things from Anhui Polytechnic University, China, in 2019. He is currently working towards a PH.D. degree with the School of Computer Science and Technology, Xidian University, China. His research interests include data mining and machine learning on graph data.
\end{IEEEbiography}

\vspace{-1cm}

\begin{IEEEbiography}[{\includegraphics[width=0.9in, height=1.25in,clip,keepaspectratio]{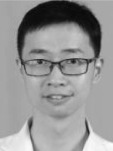}}]{Yibing Zhan} received a B.E. and a Ph.D. from the University of Science and Technology of China in 2012 and 2018, respectively. From 2018 to 2020, he was an Associate Researcher in the Computer and Software School at the Hangzhou Dianzi University. He is currently an algorithm scientist at the JD Explore Academy. His research interest includes graph neural networks and multimodal learning. He has published many papers on top conferences and journals, such as CVPR, ACM MM, AAAI, IJCV, and IEEE TMM.
\end{IEEEbiography}

\vspace{-1cm}

\begin{IEEEbiography}[{\includegraphics[width=0.9in, height=1.25in,clip,keepaspectratio]{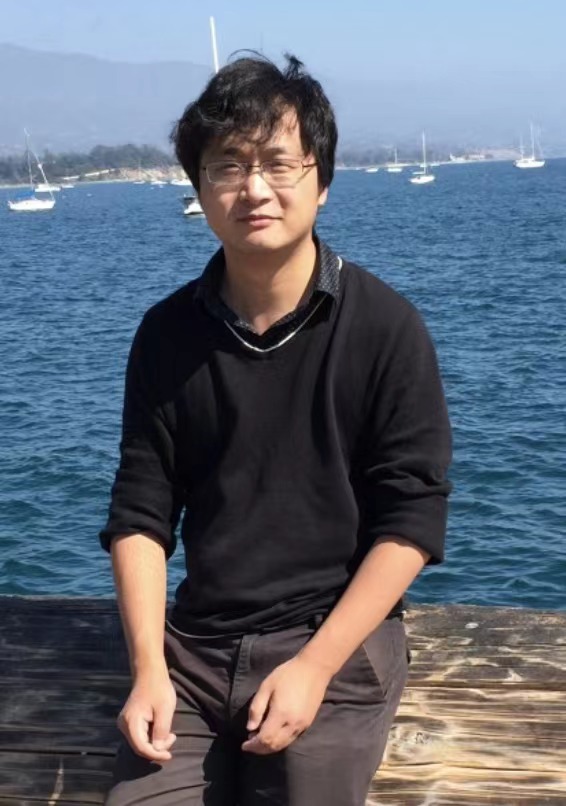}}]{Binbin Lin} is an assistant professor in the School of Software Technology at Zhejiang University, China. He received a Ph.D degree in computer science from Zhejiang University in 2012. His research interests include machine learning and decision making.
\end{IEEEbiography}

\vspace{-1cm}

\begin{IEEEbiography}[{\includegraphics[width=0.9in, height=1.25in,clip,keepaspectratio]{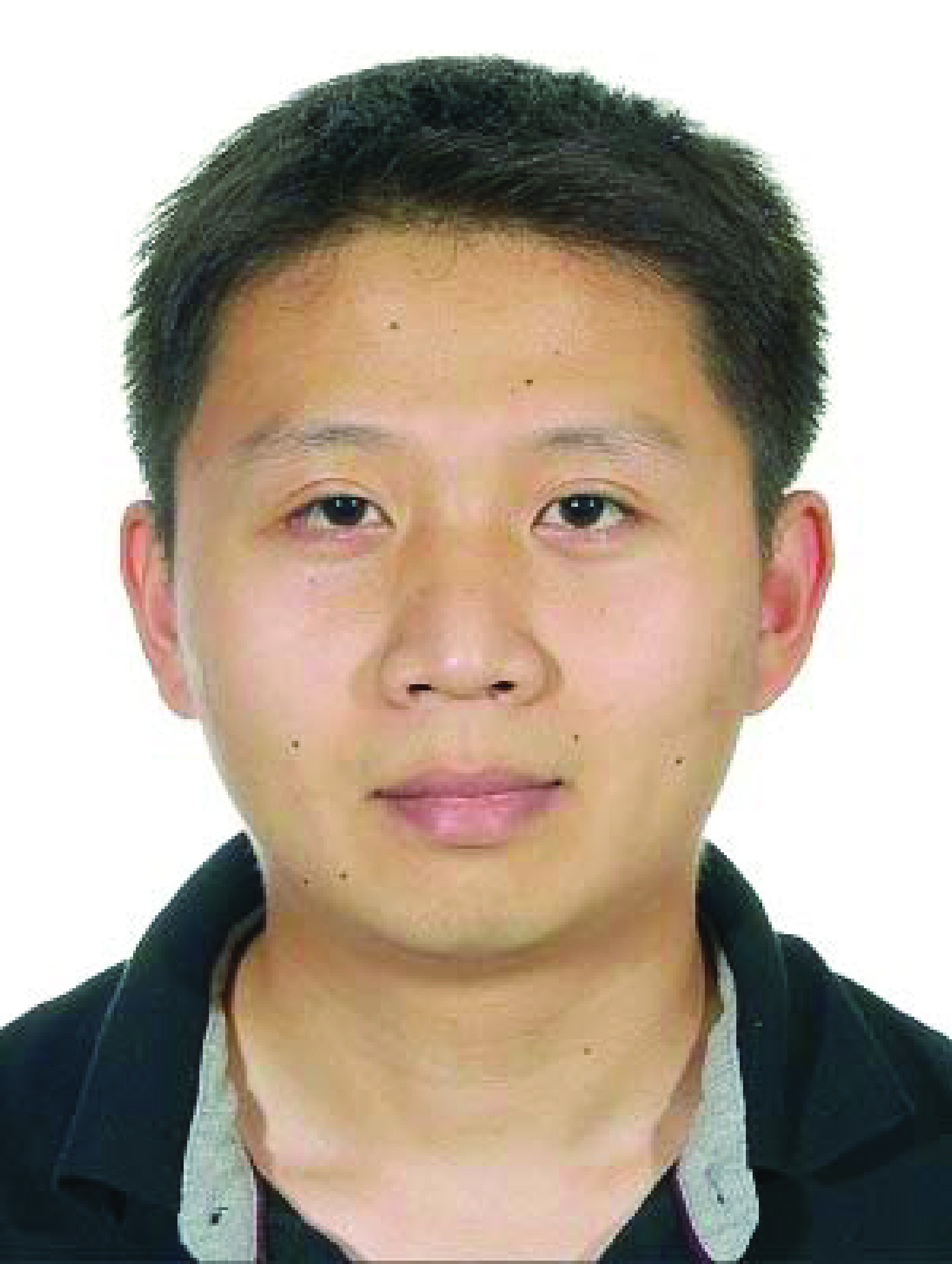}}]{Ziyu Guan} received the B.S. and Ph.D. degrees in Computer Science from Zhejiang University, Hangzhou China, in 2004 and 2010, respectively. He had worked as a research scientist in the University of California at Santa Barbara from 2010 to 2012, and as a professor in the School of Information and Technology of Northwest University, China from 2012 to 2018. He is currently a professor with the School of Computer Science and Technology, Xidian University. His research interests include attributed graph mining and search, machine learning, expertise modeling and retrieval, and recommender systems.
\end{IEEEbiography}

\vspace{-1cm}

\begin{IEEEbiography}[{\includegraphics[width=0.9in, height=1.25in,clip,keepaspectratio]{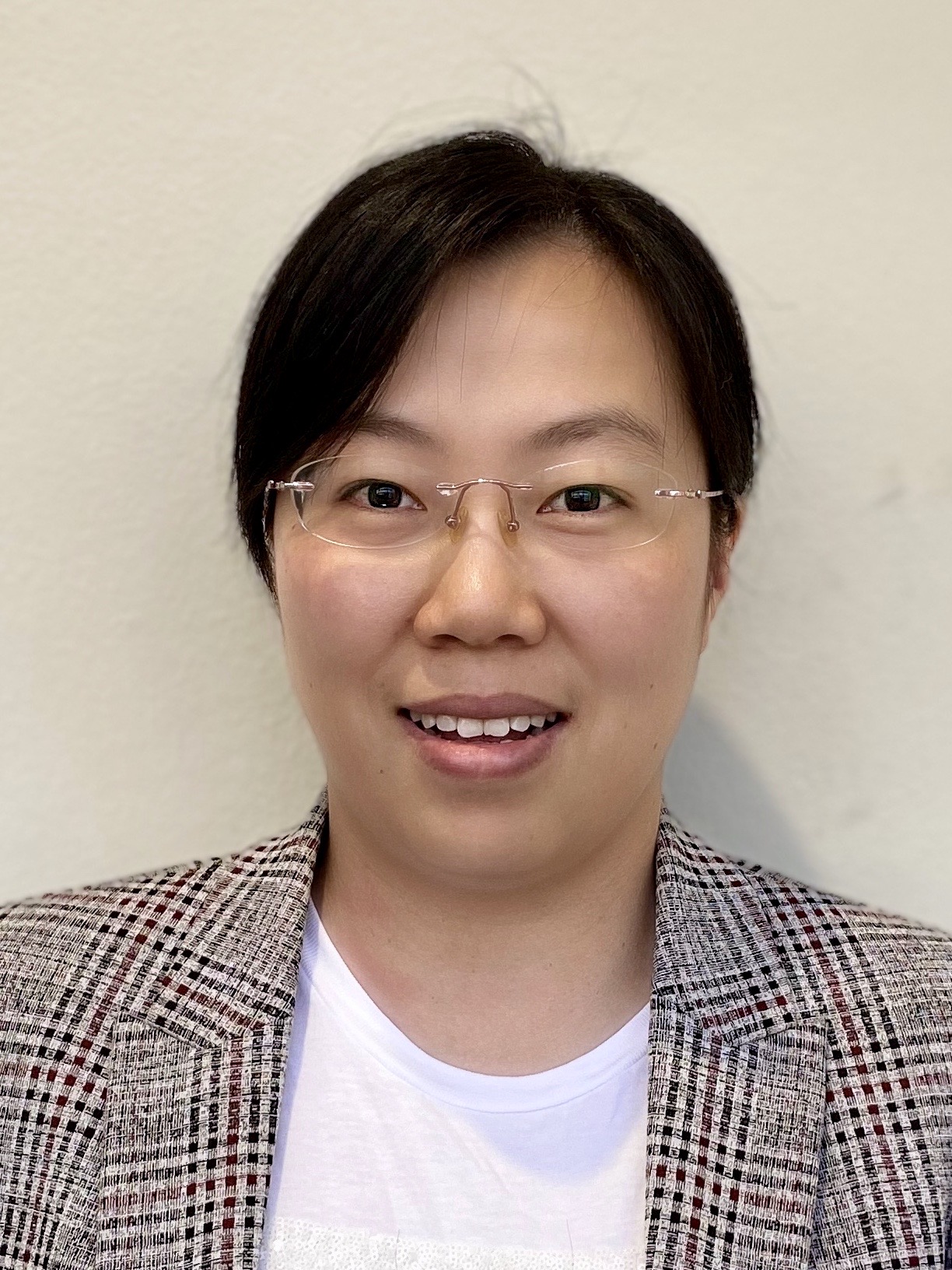}}]{Liu Liu} is currently the Research Associate in the School of Computer Science and the Faculty of Engineering at The University of Sydney, Australia. She received her Ph.D. degree from the University of Sydney. Her research interests cover stochastic optimization theory and the design of effective algorithms in machine learning and deep learning, and their applications including distributed learning, low-rank matrix, label analysis, continual learning, curriculum learning, reinforcement learning, quantum machine learning, etc. She has published more than 30+ papers on IEEE T-PAMI, T-NNLS, T-IP, T-MM, PRR, ICML, NeurIPS, CVPR, ICCV, AAAI, ICDM, etc.
\end{IEEEbiography}

\vspace{-1cm}

\begin{IEEEbiography}[{\includegraphics[width=0.9in, height=1.25in,clip,keepaspectratio]{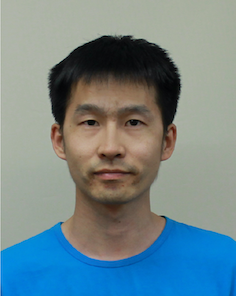}}]{Baosheng Yu} received a B.E. from the University of Science and Technology of China in 2014, and a Ph.D. from The University of Sydney in 2019. He is currently a Research Fellow in the School of Computer Science at The University of Sydney, NSW, Australia. His research interests include machine learning, computer vision, and deep learning. He has authored/co-authored 20+ publications on top-tier international conferences and journals, including TPAMI, IJCV, CVPR, ICCV and ECCV.
\end{IEEEbiography}

\vspace{-1cm}

\begin{IEEEbiography}[{\includegraphics[width=0.9in, height=1.25in,clip,keepaspectratio]{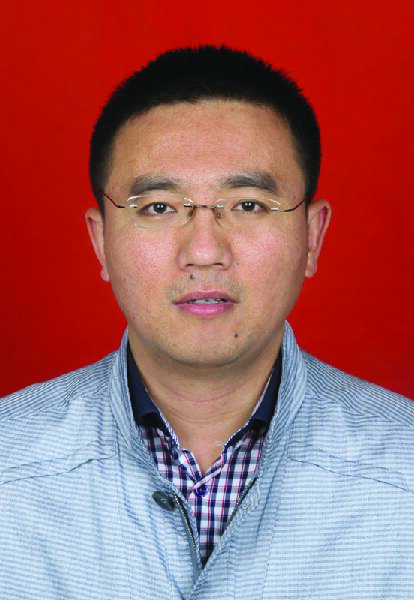}}]{Wei Zhao} received the B.S., M.S. and Ph.D. degrees from Xidian University, Xi’an, China, in 2002, 2005 and 2015, respectively. He is currently a professor in the School of Computer Science and Technology at Xidian University. His research direction is pattern recognition and intelligent systems, with specific interests in attributed graph mining and search, machine learning, signal processing and precision guiding technology.
\end{IEEEbiography}

\vspace{-1cm}

\begin{IEEEbiography}[{\includegraphics[width=0.9in, height=1.25in,clip,keepaspectratio]{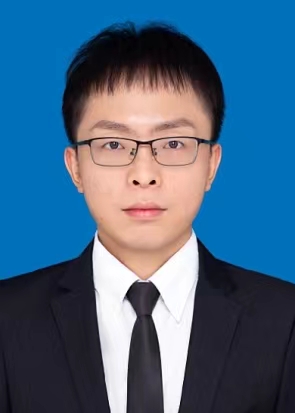}}]{Yaming Yang} received the B.S. and Ph.D. degrees in Computer Science and Technology from Xidian University, China, in 2015 and 2022, respectively. He is currently a lecturer with the School of Computer Science and Technology at Xidian University. His research interests include data mining and machine learning on graph data.
\end{IEEEbiography}

\vspace{-1cm}

\begin{IEEEbiography}[{\includegraphics[width=0.9in, height=1.25in,clip,keepaspectratio]{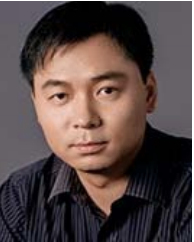}}]{Dacheng Tao} (Fellow, IEEE) is currently a Professor of computer science and an ARC Laureate Fellow with the School of Computer Science and the Faculty of Engineering, The University of Sydney, Sydney, NSW, Australia. His research results in artificial intelligence have expounded in one monograph and more than 200 publications at prestigious journals and prominent conferences, such as the IEEE TRANSACTIONS ON PATTERN ANALYSIS AND MACHINE INTELLIGENCE, the International Journal of Computer Vision, the Journal of Machine Learning Research, the Journal of Artificial Intelligence, AAAI, IJCAI, NeurIPS, ICML, CVPR, ICCV, ECCV, ICDM, and KDD, with several best paper awards. Dr. Tao is a fellow of the American Association for the Advancement of Science, the Association for Computing Machinery, The World Academy of Sciences, and the Australian Academy of Science. He received the 2018 IEEE ICDM Research Contributions Award, the 2015 and 2020 Australian Museum Eureka Prize, and the 2021 IEEE Computer Society McCluskey Technical Achievement Award.
\end{IEEEbiography}